\documentclass[11pt]{article} 

\usepackage{amsmath,amsfonts,amsthm,amssymb,color}
\usepackage{algorithm,algorithmic}
\usepackage{fullpage}

\usepackage{hyperref}
\usepackage{xspace}
\usepackage{ifthen}
\usepackage{float}
\usepackage{graphicx}
\usepackage{subfig}

\newtheorem{theorem}{Theorem}[section]

\newtheorem{lemma}[theorem]{Lemma}

\newtheorem{definition}[theorem]{Definition}

\newtheorem{remark}[theorem]{Remark}

\newcommand{\eps}{\ensuremath{\epsilon}\xspace}
\renewcommand{\tilde}{\widetilde}
\renewcommand{\hat}{\widehat}
\renewcommand{\bar}{\overline}

\newcommand{\R}{\mathbb{R}}

\newcommand{\PP}{\mathcal{P}}
\newcommand{\KK}{\mathcal{K}}
\newcommand{\CC}{\mathcal{C}}
\newcommand{\II}{\mathcal{I}}
\newcommand{\YY}{\mathcal{Y}}

\DeclareMathOperator{\tr}{tr}
\newcommand{\norm}[1]{\left\|#1\right\|}
\newcommand{\normzero}[1]{\norm{#1}_0}
\newcommand{\normone}[1]{\norm{#1}_1}
\newcommand{\normtwo}[1]{\norm{#1}_2}
\newcommand{\normtwok}[1]{\norm{#1}_{2,k}}
\newcommand{\norminf}[1]{\norm{#1}_\infty}
\newcommand{\fnorm}[1]{{\norm{#1}}_F}
\newcommand{\fkknorm}[1]{{\norm{#1}}_{F,k,k}}
\newcommand{\fksnorm}[1]{{\norm{#1}}_{F,k^2}}
\newcommand{\ftwoksnorm}[1]{{\norm{#1}}_{F,2k^2}}
\newcommand{\kftwonorm}[1]{{\norm{#1}}_{\star,2}}

\providecommand{\expect}[2]{\ensuremath{\ifthenelse{\equal{#1}{}}{\mathbb{E}}{\mathbb{E}_{#1}}\!\left[#2\right]}\xspace}
\providecommand{\prob}[2]{\ensuremath{\ifthenelse{\equal{#1}{}}{\Pr}{\Pr_{#1}}\!\left[#2\right]}\xspace}

\newenvironment{lp*}{\begin{equation*}  \begin{array}{lll}}{\end{array}\end{equation*}}

\usepackage{enumitem}

\newcommand{\NN}{\mathcal{N}}
\newcommand{\abs}[1]{\left|{#1}\right|}
\DeclareMathOperator{\diag}{diag}

\newcommand{\littlesum}{\mathop{\textstyle \sum}}

\newcommand{\mus}{\mu}

\title{Outlier-Robust Sparse Estimation via Non-Convex Optimization\footnotetext{An implementation of our algorithms is available at \url{https://github.com/guptashvm/Sparse-GD}.}}
\author{
\begin{tabular}{c c}
  \begin{tabular}{c}
    Yu Cheng\thanks{\texttt{yu\_cheng@brown.edu}. Supported in part by NSF Award CCF-2122628. Some of this work was done while the author was at the University of Illinois at Chicago.}\\
    Brown University
  \end{tabular} &
  \begin{tabular}{c}
    Ilias Diakonikolas\thanks{\texttt{ilias@cs.wisc.edu}. Supported by NSF Medium Award CCF-2107079, NSF Award CCF-1652862 (CAREER), NSF Award AiTF-2006206, a Sloan Research Fellowship, and a DARPA Learning with Less Labels (LwLL) grant.}\\
    University of Wisconsin-Madison
  \end{tabular} \\ \\
  \begin{tabular}{c}
    Rong Ge\thanks{\texttt{rongge@cs.duke.edu}. Supported by NSF Award CCF-1704656, NSF Award CCF-1845171 (CAREER), NSF Award CCF-1934964, a Sloan Research Fellowship, and a Google Faculty Research Award.}\\
    Duke University
  \end{tabular} &
  \begin{tabular}{c}
    Shivam Gupta\thanks{\texttt{shivamgupta@utexas.edu}. Supported by NSF Award CCF-2008868, NSF Award AiTF-2006206, and the NSF AI Institute for Foundations of Machine Learning (IFML). Some of this work was done while the author was visiting the University of Wisconsin-Madison.} \\
    University of Texas at Austin
  \end{tabular} \\ \\
  \begin{tabular}{c}
    Daniel M. Kane\thanks{\texttt{dakane@cs.ucsd.edu}. Supported by NSF Award CCF-1652862 (CAREER), NSF Award CCF-2107547, a Sloan Research Fellowship, and a grant from CasperLabs.} \\
    University of California, San Diego
  \end{tabular} &
  \begin{tabular}{c}
    Mahdi Soltanolkotabi\thanks{\texttt{soltanol@usc.edu}. Supported by the Packard Fellowship in Science and Engineering, a Sloan Research Fellowship, NSF Award CCF-1846369 (CAREER), NSF Award CCF-1813877, DARPA Learning with Less Labels (LwLL) and Fast Network Interface Cards (FastNICs) grants, and Faculty Research Awards from Google and Amazon.}\\
    University of Southern California
  \end{tabular}
\end{tabular}
}
\date{}

\begin{document}
\maketitle


\begin{abstract}
We explore the connection between outlier-robust high-dimensional statistics and non-convex optimization in the presence of sparsity constraints, with a focus on the fundamental tasks of robust sparse mean estimation and robust sparse PCA. We develop novel and simple optimization formulations for these problems 
such that {\em any} approximate stationary point of the associated optimization problem 
yields a near-optimal solution for the underlying robust estimation task. As a corollary, we obtain that 
any first-order method that efficiently converges to stationarity yields an efficient algorithm for these tasks. 
The obtained algorithms are simple, practical, and succeed under broader distributional assumptions compared to prior work.
\end{abstract}

\setcounter{page}{0}

\thispagestyle{empty}

\newpage


\section{Introduction} \label{sec:intro}

In several modern machine learning (ML) applications, 
such as ML security~\cite{Barreno2010,BiggioNL12, SteinhardtKL17, DiakonikolasKKLSS2018sever} 
and exploratory analysis of real datasets, e.g., in population genetics~\cite{RP-Gen02, Pas-MG10, Li-Science08, DKK+17},
typical datasets contain a non-trivial fraction of arbitrary (or even adversarial) outliers.
Robust statistics~\cite{HampelEtalBook86, Huber09} is the subfield of statistics
aiming to design estimators that are tolerant to a {\em constant fraction} of outliers, independent
of the dimensionality of the data. Early work in this field, see, e.g.,~\cite{Tukey60, Huber64, Tukey75} 
developed sample-efficient robust estimators for various basic tasks, alas with runtime exponential in the dimension. 

During the past five years, a line of work in computer science, starting with~\cite{DKKLMS16, LaiRV16}, 
has developed the first {\em computationally efficient} robust high-dimensional estimators 
for a range of tasks. This progress has led to a revival of robust statistics from an algorithmic perspective
(see, e.g.,~\cite{DK19-survey, DKKL+21} for surveys on the topic).
In this work, we focus on high-dimensional estimation tasks in the presence
of sparsity constraints.
To rigorously study these problems, we need to formally define the model of data corruption.
Throughout this work, we work with the following standard contamination model.

\begin{definition}[Strong Contamination Model, see~\cite{DKKLMS16}] \label{def:adv}
Given a parameter $0< \eps < 1/2$ and a distribution family $\mathcal{D}$ on $\R^d$,
the \emph{adversary} operates as follows: The algorithm specifies a
number of samples $n$, and $n$ samples are drawn from some unknown $D \in \mathcal{D}$.
The adversary is allowed to inspect the samples, remove up to $\eps n$ of them
and replace them with arbitrary points. This modified set of $n$ points is then given as input
to the algorithm. We say that a set of samples is {\em $\eps$-corrupted}
if it is generated by the above process.
\end{definition}


High-dimensional robust statistics is algorithmically challenging 
because the natural optimization formulations of such tasks are typically non-convex. 
The recent line of work on algorithmic robust statistics has led to a range of
sophisticated algorithms. In some cases, such algorithms require solving large convex relaxations, 
rendering them computationally prohibitive for large-scale problems.
In other cases, they involve a number of hyper-parameters that may require careful tuning.
Motivated by these shortcomings of known algorithms, recent work~\cite{ChengDGS20, ZJS20} 
established an intriguing connection between high-dimensional robust estimation 
and non-convex optimization. The high-level idea is quite simple:
Even though typical robust statistics tasks lead to non-convex formulations, it may still be possible to leverage
the underlying structure to show that standard first-order methods provably and efficiently reach near-optimal
solutions. Indeed, \cite{ChengDGS20, ZJS20} were able to prove such statements
for robust mean estimation under natural distributional assumptions. Specifically,
these works established that any (approximate) stationary point of a well-studied non-convex formulation for 
robust mean estimation yields a near-optimal solution for the underlying robust estimation task.

In this work, we continue this line of work with a focus on {\em sparse} estimation tasks.
Leveraging sparsity in high-dimensional datasets is a fundamental problem of significant practical importance.\nocite{TorreB01}
Various formalizations of this problem have been investigated in statistics and machine learning
for at least the past two decades (see, e.g.,~\cite{Hastie15} for a textbook on the topic).
We focus on {\em robust sparse mean estimation} and {\em robust sparse PCA}.
Sparse mean estimation is arguably one of the most
fundamental sparse estimation tasks and is closely related to the Gaussian sequence model~\cite{Tsybakov08, J17}.
The task of sparse PCA in the spiked covariance model, initiated in~\cite{J01}, has been extensively
investigated (see Chapter~8 of~\cite{Hastie15} and references therein).

In the context of robust sparse mean estimation, we are given an $\eps$-corrupted set of samples 
from a distribution with unknown mean $\mu \in \R^d$ where $\mu$ is $k$-sparse, and we want to compute 
a vector $\widehat{\mu}$ close to $\mu$.
In the context of robust sparse PCA (in the spiked covariance model), we are given an $\eps$-corrupted set of samples from
a distribution with covariance matrix $I+ \rho vv^T$, where $v \in \R^d$ is $k$-sparse
and the goal is to approximate $v$.
It is worth noting that for both problems, we have access to much fewer samples 
compared to the non-sparse case (roughly $O(k^2 \log d)$ instead of $\Omega(d)$). 
Consequently, the design and analysis of optimization formulations for robust sparse estimation requires new ideas and techniques that significantly deviate from the standard (non-sparse) case.


\subsection{Our Results and Contributions}
We show that standard first-order methods lead to robust and efficient algorithms for sparse mean estimation and sparse PCA.
Our main contribution is to propose novel (non-convex) formulations for these robust estimation tasks, and to show that {\em approximate stationarity suffices for near-optimality}.
We establish landscape results showing that \emph{any} approximate stationary point of our objective function yields a near-optimal solution for the underlying robust estimation task.
Consequently, gradient descent (or any other methods converging to stationarity) can solve these problems.

Our results provide new insights and techniques in designing and analyzing (non-convex) optimization formulations of robust estimation tasks.
Our formulations and structural results immediately lead to simple and practical algorithms for robust sparse estimation.
Importantly, the gradient of our objectives can be computed efficiently via a small number of basic matrix operations.
In addition to their simplicity and practicality, our methods provably succeed under more general distributional assumptions compared to prior work. 

For robust sparse mean estimation and robust sparse PCA, our landscape results require deterministic 
conditions on the original set of good samples. We refer to these conditions as \emph{stability conditions} (Definitions~\ref{def:sparse-stability}~and~\ref{def:sparse-pca-stability}, formally defined in Section~\ref{sec:prelim}).
At a high level, they state that the first and second moments of a set of samples are stable when {\em any} $\eps$-fraction of the samples are removed.
These stability conditions hold with high probability for a set of clean samples drawn from natural families of distributions (e.g., subgaussian).

For robust sparse mean estimation, we establish the following result.
\begin{theorem}[Robust Sparse Mean Estimation]
\label{thm:intro-sparse-mean}
Let $0 < \eps < \eps_0$ for some universal constant $\eps_0$ and let $\delta > \eps$.
Let $G^\star$ be a set of $n$ samples that is $(k,\eps,\delta)$-stable (per Definition~\ref{def:sparse-stability}) w.r.t. a distribution with unknown $k$-sparse mean $\mu \in \R^d$.
Let $S = (X_i)_{i=1}^n$ be an $\eps$-corrupted version of $G^\star$.~\footnote{For two sets of samples $S$ and $T$, we say $S$ is an $\eps$-corrupted version of $T$ if $|S| = |T|$ and $|S \setminus T| \le \eps |S|$.}
There is an algorithm that on inputs $S$, $k$, $\eps$, and $\delta$, runs in polynomial time and returns a $k$-sparse vector $\hat \mu \in \R^d$ such that $\normtwo{\hat \mu - \mu} \le O(\delta)$.
\end{theorem}


We emphasize that a key novelty of Theorem~\ref{thm:intro-sparse-mean} is
that the underlying algorithm is a {\em first-order method} applied
to a {\em novel non-convex formulation} of the problem. The major advantage of our algorithm 
over prior work~\cite{BDLS17, DKKPS19-sparse} is its simplicity, practicality, and the fact that
it seamlessly applies to a wider class of distributions on the clean data.

As we will discuss in Section~\ref{sec:sparse-mean}, when the ground-truth distribution $D$ is subgaussian 
with unknown $k$-sparse mean $\mu \in \R^d$ and identity covariance, a set of 
$n = \tilde \Omega(k^2 \log d/\eps^2)$ samples drawn from $D$ is $(k,\eps,\delta)$-stable (Definition~\ref{def:sparse-stability}) with high probability 
for $\delta = O(\eps\sqrt{\log(1/\eps)})$. 
It follows as an immediate corollary of Theorem~\ref{thm:intro-sparse-mean} that, 
given as input an $\eps$-corrupted set of samples, we can compute a vector $\hat \mu$ 
that is $O(\delta) = O(\eps\sqrt{\log(1/\eps)})$ close to the true mean $\mu$.
Our sample complexity upper bound matches the known computational-statistical 
lower bounds~\cite{DKS17-sq, BrennanB20} for this problem. 
More generally, one can relax the concentration assumption on the clean data 
and obtain qualitatively similar error guarantees.

Next we state our main result for robust sparse PCA.

\begin{theorem}[Robust Sparse PCA] \label{thm:intro-sparse-pca-struc}
Let $0 < \rho \le 1$ and $0 < \eps < \eps_0$ for some universal constant $\eps_0$.
Let $G^\star$ be a set of $n$ samples that is $(k,\eps,\delta)$-stable (as in Definition~\ref{def:sparse-pca-stability}) w.r.t. a centered distribution with covariance $\Sigma = I + \rho v v^\top$, for an unknown $k$-sparse unit vector $v \in \R^d$. Let $S = (X_i)_{i=1}^n$ be an $\eps$-corrupted version of $G^\star$.
There is an algorithm that on inputs $S$, $k$, and $\eps$, runs in polynomial time and returns a unit vector $u\in \R^d$ such that $\fnorm{uu^\top - vv^\top} = O(\sqrt{\delta/\rho})$.
\end{theorem}

Interestingly, our algorithm for robust sparse PCA is a first-order method applied 
to a simple {\em convex} formulation of the problem. We view the existence of a convex formulation
as an intriguing fact that, surprisingly, was not observed in prior work.

As we will discuss in Section~\ref{sec:sparse-pca}, when the ground-truth distribution $D$ is centered subgaussian 
with covariance $\Sigma = I + \rho vv^\top$, for an unknown $k$-sparse unit vector $v \in \R^d$,
a set of $n = \tilde \Omega(k^2 \log d/\eps^2)$ samples drawn from $D$ is $(k,\eps,\delta)$-stable (Definition~\ref{def:sparse-pca-stability}) with high probability for $\delta = O(\eps\log(1/\epsilon))$.
Therefore, our algorithm outputs a vector that is $O(\sqrt{\eps\log(1/\eps)/\rho})$ close to the true direction $v$.
The sample complexity in this case
nearly matches the computational-statistical lower bound of $\Omega(k^2 \log d / \eps^2)$~\cite{BerthetR13} 
which holds even without corruptions. 
While the error guarantee of our algorithm 
is slightly worse compared to prior work~\cite{BDLS17, DKKPS19-sparse} for Gaussian data (we get $O(\sqrt{\delta/\rho})$ rather than $O(\delta/\rho)$), 
we note that our algorithm is a first-order method applied to a simple formulation and it works for a broader family of distributions.


\paragraph{Prior Work on Robust Sparse Estimation.}
We provide a detailed summary of prior work for comparison.
\cite{BDLS17} obtained the first sample-efficient and polynomial-time
algorithms for robust sparse mean estimation and robust sparse PCA. 
These algorithms succeed for Gaussian inliers and inherently use 
the ellipsoid method. However, the separation oracle
required for the ellipsoid method turns out to be another convex program --- corresponding
to an SDP to solve sparse PCA. As a consequence, the running time of these algorithms,
while polynomially bounded, is impractically high.
\cite{LLC19} proposed an algorithm for robust sparse mean estimation via iterative trimmed hard thresholding, 
which can only tolerate a {\em sub-constant} fraction of corruptions.
\cite{DKKPS19-sparse} gave iterative spectral robust algorithms for sparse mean estimation and sparse PCA.
These algorithms are still quite complex and are only shown to succeed under Gaussian inliers. 

\subsection{Overview of Our Approach} \label{ssec:techniques}
In this section, we give an overview of our approach for robust sparse mean estimation.
At a high level, we want to assign a nonnegative weight to each data point such that the weighted empirical mean is close to the true mean.
The constraint on the weight vector is that it represents at least a (fractional) set of $(1-\eps)$-portion of the input dataset.
Formally, given $n$ datapoints $(X_i)_{i=1}^n$, the goal is to find a weight vector $w \in \R^n$ such that $\mu_w = \sum_i w_i X_i$ is close to the true mean $\mu$.
The constraint on $w$ is that it belongs to
\[
\Delta_{n,\eps} = \left\{ w \in \R^n : \normone{w}=1 \text{ and } 0 \le w_i \le \tfrac{1}{(1-\eps) n} \; \forall i \right\},
\]
which is the convex hull of all uniform distributions over subsets $S \subseteq [n]$ of size $|S| = (1-\eps)n$.

Let $\Sigma_w = \sum_i w_i (X_i - \mu_w)(X_i - \mu_w)^\top$ denote the weighted empirical covariance matrix.
It is well-known that if one can find $w \in \Delta_{n,\eps}$ that minimizes the weighted empirical variance $v^\top \Sigma_w v$ for all $k$-sparse unit vectors $v$, then $\mu_w$ must be close to $\mu$.
Unfortunately, it is NP-Hard to find the sparse direction $v$ with the largest variance.
To get around this issue,~\cite{BDLS17} considered the following convex relaxation, minimizing the variance for convex combinations of sparse directions:
\begin{align}
\label{eqn:formulation-bdls}
& \min_w \; \max_{\tr(A)=1,\sum_{ij}\abs{A_{ij}}\le k,A \succeq 0} (A \bullet \Sigma_w) \; .
\end{align}
Given $w$, the optimal $A$ can be found using semidefinite programming (SDP).
\cite{ZJS20} observed that any stationary point $w$ of~\eqref{eqn:formulation-bdls} 
gives a good solution for robust sparse mean estimation. However, solving~\eqref{eqn:formulation-bdls} 
requires convex programming to compute the gradient in each iteration. As explained in the proceeding
discussion, our approach circumvents this shortcoming, leading to a formulation 
for which each gradient can be computed \emph{using only basic matrix operations}.

In this work, we propose and analyze the following optimization formulation:
\begin{equation*}
\min_w \; f(w) = \fkknorm{\Sigma_w - I} \quad \textrm{ subject to } \; w \in \Delta_{n,\eps} \; ,
\end{equation*}
where $\fkknorm{A}$ is the Frobenius norm of the $k^2$ entries of $A$ with largest magnitude, with the additional constraint that these $k^2$ entries are chosen from $k$ rows with $k$ entries in each row.

We prove that any stationary point of $f(w)$ yields a good solution for robust sparse mean estimation.
Here we provide a brief overview of our proof (see Section~\ref{sec:sparse-mean} for more details).
Given a weight vector $w$, we show that if $w$ is not a good solution, then moving toward $w^\star$ 
(the weight vector corresponding to the uniform distribution on the clean input samples) 
will decrease the objective value.
Formally, we will show that, for any $0 < \eta < 1$,
\begin{align*}
\Sigma_{(1-\eta)w + \eta w^\star} &= (1-\eta)\Sigma_w + \eta\Sigma_{w^\star}
  + \eta(1-\eta)(\mu_w - \mu_{w^\star})(\mu_w - \mu_{w^\star})^\top \; .
\end{align*}

We can then take $\|\cdot\|_{F,k,k}$ norm on both sides (after subtracting $I$) and show that the third term can be essentially ignored.
If the third term were not there, we would have
\begin{align*}
f((1-\eta)w + \eta w^\star)
&= \fkknorm{\Sigma_{(1-\eta)w + \eta w^\star}-I} \\
&\le (1-\eta)\fkknorm{\Sigma_w-I} + \eta\fkknorm{\Sigma_{w^\star}-I}
= (1-\eta)f(w) + \eta f(w^\star) \; .
\end{align*}
Therefore, if $w$ is a bad solution with $f(w)$ much larger than $f(w^\star)$, then $w$ cannot be a stationary point because $f$ decreases when we move from $w$ to $(1-\eta)w + \eta w^\star$.

\begin{remark}
{\em The technical overview for robust sparse PCA follows a similar high-level approach, 
but is somewhat more technical. It is deferred to Section~\ref{sec:sparse-pca}.}
\end{remark}


\paragraph{Roadmap.}
In Section~\ref{sec:prelim}, we introduce basic notations and the deterministic stability conditions that we require on the good samples. We present our algorithms and analysis for robust sparse mean estimation in Section~\ref{sec:sparse-mean} and robust sparse PCA in Section~\ref{sec:sparse-pca}. In Section~\ref{sec:experiment}, we evaluate our algorithm on synthetic datasets and show that it achieves good statistical accuracy under various noise models.


\section{Preliminaries and Background} \label{sec:prelim}

\paragraph{Notation.}
For a positive integer $n$, let $[n]=\{1, \ldots, n\}$.
For a vector $v$, we use $\normzero{v}$, $\normone{v}$, $\normtwo{v}$, and $\norminf{v}$ for the number of non-zeros, 
the $\ell_1$, $\ell_2$, and $\ell_\infty$ norm of $v$ respectively.
Let $I$ be the identity matrix.
For a matrix $A$, we use $\normtwo{A}$, $\fnorm{A}$, $\tr(A)$ for the spectral norm, Frobenius norm, and trace of $A$ respectively.
For two vectors $x, y$, let $x^\top y$ denote their inner product.
For two matrices $A, B$, we use $A \bullet B = \tr(A^\top B)$ for their entrywise inner product.
A matrix $A$ is said to be positive semidefinite (PSD) if $x^\top A x \ge 0$ for all $x$.
We write $A \preceq B$ iff $(B - A)$ is PSD.

For a vector $w \in \R^n$, let $\diag(w) \in \R^{n \times n}$ denote a diagonal matrix with $w$ on the diagonal.
For a matrix $A \in \R^{n \times n}$, let $\diag(A) \in \R^n$ denote a column vector with the diagonal of $A$.
For a vector $v \in \R^d$ and a set $S \subseteq [d]$, we write $v_S \in \R^d$ for a vector that is equal to $v$ on $S$ and zero everywhere else.
Similarly, for a matrix $A \in \R^{d\times d}$ and a set $S \subseteq ([d] \times [d])$, we write $A_S$ for a matrix that is equal to $A$ on $S$ and zero everywhere else.

For a vector $v$, we define $\normtwok{v} = \max_{|S|=k} \normtwo{v_S}$ to be the maximum $\ell_2$-norm of any $k$ entries of $v$.
For a matrix $A$, we define $\fksnorm{A}$ to be the maximum Frobenius norm of any $k^2$ entries of $A$.
Moreover, we define $\fkknorm{A}$ to be the maximum Frobenius norm of any $k^2$ entries with the extra requirement that these entries must be chosen from $k$ rows with $k$ entries in each row.
Formally,
\begin{align}
\label{eqn:def-fks-fkk}
\fksnorm{A} &= \max_{|Q|=k^2} \fnorm{A_Q} \quad \text{ and } \quad
\fkknorm{A}^2 = \max_{|S|=k} \littlesum_{i\in S}{\normtwok{A_i}^2} \text{ where $A_i$ is $i$-th row of $A$} \; .
\end{align}


\paragraph{Sample Reweighting Framework.}
We use $n$ for the number of samples, $d$ for the dimension, and $\eps$ for the fraction of corrupted samples.
For sparse estimation, we use $k$ for the sparsity of the ground-truth parameters.
We use $G^\star$ for the original set of $n$ good samples.
We use $S = G \cup B$ for the input samples after the adversary replaced $\eps$-fraction of $G^\star$, 
where $G \subset G^\star$ is the set of remaining good samples and $B$ 
is the set of bad samples (outliers) added by the adversary.
Note that $|G| = (1-\eps) n$ and $|B| = \eps n$.

Given $n$ samples $X_1, \ldots, X_n$, we write $X \in \R^{d \times n}$ as the sample matrix where the $i$-th column is $X_i$.
For a weight vector $w \in \R^n$, we use $\mu_w = X w = \sum_i w_i X_i$ for the weighted empirical mean and $\Sigma_w = X \diag(w) X - \mu_w \mu_w^\top = \sum_i w_i (X_i - \mu_w)(X_i - \mu_w)^\top$ for the weighted empirical covariance.
Let $\Delta_{n,\eps}$ be the convex hull of all uniform distributions over subsets $S \subseteq [n]$ of size $|S| = (1-\eps)n$:
$\Delta_{n,\eps} = \{ w \in \R^n : \normone{w}=1 \text{ and } 0 \le w_i \le \tfrac{1}{(1-\eps) n} \; \forall i \}$,
In other words, every $w \in \Delta_{n,\eps}$ corresponds to a fractional set of $(1-\eps)n$ samples.
We use $w^\star$ to denote the uniform distribution on $G$ (the remaining good samples in $S$).

\paragraph{Deterministic Stability Conditions.}
For robust sparse mean estimation and robust sparse PCA, we require the following conditions on the original set of good samples respectively.

\begin{definition}[Stability Conditions for Sparse Mean]
\label{def:sparse-stability}
A set of $n$ samples $G^\star = (X_i)_{i=1}^n$ is said to be $(k,\eps,\delta)$-stable (w.r.t. a distribution with mean $\mu$) iff for any weight vector $w \in \Delta_{n,2\eps}$, we have
$\normtwok{\mu_w - \mu} \le \delta$ and 
$\fkknorm{\Sigma_w - I} \le \delta^2/\eps$,
where $\mu_w$ and $\Sigma_w$ are the weighted empirical mean and covariance matrix respectively, and the $\fkknorm{\cdot}$ norm is defined in Equation~\eqref{eqn:def-fks-fkk}.
\end{definition}

\begin{definition}[Stability Conditions for Sparse PCA]
\label{def:sparse-pca-stability}
A set of $n$ samples $G^\star = (X_i)_{i=1}^n$ is $(k,\eps,\delta)$-stable 
(w.r.t. a centered distribution with covariance $I + \rho v v^\top$) iff for any weight vector $w \in \Delta_{n,2\eps}$,
${ \ftwoksnorm{M_w - (I + \rho v v^\top)} \le \delta \; ,}$
where $M_w = \sum_i w_i X_i X_i^\top$ and the $\ftwoksnorm{\cdot}$ norm is defined in Equation~\eqref{eqn:def-fks-fkk}.
\end{definition}

\paragraph{First-Order Stationary Points.}
We give a formal definition of the notion of (approximate) first-order stationary point that we use in this paper.

\begin{definition}[Approximate Stationary Points]
\label{def:stationary}
Fix a convex set $\KK$ and a differentiable function $f$. For $\gamma \ge 0$, we say that $x \in \KK$ is a $\gamma$-stationary point of $f$ iff the following condition holds: For any unit vector $u$ where $x + \alpha u \in \KK$ for some $\alpha > 0$, we have $u^\top \nabla f(x) \ge -\gamma$.
\end{definition}
We note that the objective functions studied in this paper are not everywhere differentiable.
This is because, taking the $\fkknorm{\cdot}$ norm as an example, there can be ties in choosing the largest $k^2$ entries.
When the function $f$ is not differentiable, we use $\nabla f$ informally to denote an element of the sub-differential.
We will show in Appendix~\ref{apx:optimization} that, while $f$ is not differentiable, it does have a nonempty subdifferential, 
as it can be written as the pointwise maximum of differentiable functions. 


\section{Robust Sparse Mean Estimation}
\label{sec:sparse-mean}

In this section, we present our non-convex approach for robust sparse mean estimation.
We will optimize the following objective, where $\fkknorm{\cdot}$ is defined in Equation~\eqref{eqn:def-fks-fkk}:
\begin{equation}
\label{eqn:sparse-obj-fks}
\min_w \; f(w) = \fkknorm{\Sigma_w - I} \quad \textrm{ subject to } \; w \in \Delta_{n,\eps} \; .
\end{equation}

We will show that the objective function~\eqref{eqn:sparse-obj-fks} has no bad stationary points (Theorem~\ref{thm:sparse-mean-struc}).
In other words, {\em every} first-order stationary point of $f$ yields a good solution for robust sparse mean estimation.

Our algorithm is stated in Algorithm~\ref{alg:sparsemean}.
As a consequence of our landscape result (Theorem~\ref{thm:sparse-mean-struc}), we know that Algorithm~\ref{alg:sparsemean} works {\em no matter how} we find a stationary point of $f$ (because any stationary point works), so we intentionally did not specify how to find such a point.
As a simple illustration, we show that (projected) gradient descent can be used to minimize $f$.
The convergence analysis and iteration complexity are provided in Appendix~\ref{apx:optimization}.

\begin{algorithm}[h]
\caption{Robust sparse mean estimation.}\label{alg:sparsemean}
{\bfseries Input:} $k > 0$, $0 < \eps < \eps_0$, and an $\eps$-corrupted set of samples $(X_i)_{i=1}^n$ drawn from a distribution with $k$-sparse mean $\mu$.~\footnotemark \\
{\bfseries Output:} a vector $\hat{\mu}$ that is close to $\mu$.

\begin{algorithmic}[1]
\STATE Find a first-order stationary point $w \in \Delta_{n,\eps}$ of the objective $\min_w f(w) = \fkknorm{\Sigma_w - I}$.
\STATE Return $\hat{\mu} = (\mu_w)_Q$ where $Q$ is a set of $k$ entries of $\mu_w$ with largest magnitude.
\end{algorithmic}
\end{algorithm}
\footnotetext{Without loss of generality we can assume that $\eps$ is given to the algorithm.  This is because we can run a binary search to determine $\eps$: if our guess of $\eps$ is too small, then the algorithm will output a $w$ whose objective value $f(w)$ is much larger than it should be.} 

Formally, we first prove that Algorithm~\ref{alg:sparsemean} can output a vector $\hat \mu \in \R^d$ that is close to $\mu$ in $\normtwok{\cdot}$ norm, as long as the good samples satisfies the stability condition in Definition~\ref{def:sparse-stability}.

\begin{theorem}
\label{thm:sparse-mean-struc}
Fix $k > 0$, $0 < \eps < \eps_0$, and $\delta > \eps$.
Let $G^\star$ be a set of $n$ samples that is $(k,\eps,\delta)$-stable (as in Definition~\ref{def:sparse-stability}) w.r.t. a distribution with unknown $k$-sparse mean $\mu \in \R^d$.
Let $S = (X_i)_{i=1}^n$ be an $\eps$-corrupted version of $G^\star$.
Let $f(w) = \fkknorm{\Sigma_w - I}$.
Let $\gamma = O(n^{1/2} \delta^2 \eps^{-3/2})$.
Then, for any $w \in \Delta_{n,\eps}$ that is a $\gamma$-stationary point of $f(w)$, we have $\normtwok{\mu_w - \mu} = O(\delta)$.
\end{theorem}

Once we have a vector $\mu_w$ that is $O(\delta)$-close to $\mu$ in $\normtwok{\cdot}$ norm, we can guarantee that a truncated version of $\mu_w$ (the output $\hat \mu$ of Algorithm~\ref{alg:sparsemean}) is $O(\delta)$-close to $\mu$ in the $\ell_2$-norm:
\begin{lemma}
\label{lem:thres-normtwok}
Fix two vectors $x, y$ with $\normzero{x} \le k$ and $\normtwok{x-y} \le \delta$.
Let $z$ be a vector that keeps the $k$ entries of $y$ with largest absolute values and sets the rest to $0$.
We have $\normtwo{x-z} \le \sqrt{5} \delta$.
\end{lemma}

Theorem~\ref{thm:intro-sparse-mean} follows immediately from Theorem~\ref{thm:sparse-mean-struc} and Lemma~\ref{lem:thres-normtwok}.

We can apply Theorem~\ref{thm:intro-sparse-mean} to get an end-to-end result for subgaussian distributions.
We show that the required stability conditions are satisfied with a small number of samples (Lemma~\ref{lem:sparse-mean-stability}).

\begin{lemma}
\label{lem:sparse-mean-stability}
Fix $k > 0$ and $0 < \eps < \eps_0$. Let $G^\star$ be a set of $n$ samples that are drawn i.i.d. from a subgaussian distribution with mean $\mu$ and covariance $I$. If $n = \Omega(k^2\log d/\eps^2)$, then with probability at least $1-\exp(-\Omega(k^2\log d))$, $G^\star$ is $(k,\eps, \delta)$-stable (as in Definition~\ref{def:sparse-stability}) for $\delta = O(\eps\log(1/\eps))$.
\end{lemma}

Combining Theorem~\ref{thm:intro-sparse-mean} and Lemma~\ref{lem:sparse-mean-stability}, we know that given as input an $\eps$-corrupted set of $O(k^2\log d /\eps^2)$ samples drawn from a subgaussian distribution with $k$-sparse mean $\mu$, the output of Algorithm~\ref{alg:sparsemean} is $O(\eps\sqrt{\log(1/\eps)})$-close to $\mu$ in $\ell_2$-norm.

In the rest of this section, we will prove Theorem~\ref{thm:sparse-mean-struc}.
All omitted proofs in this section are deferred to Appendix~\ref{apx:sparse-mean}.

We start with some intuition on why we choose our objective function~\eqref{eqn:sparse-obj-fks}.
We would like to design $f(w) = g(\Sigma_w-I)$ to satisfy the following properties:
\begin{enumerate}[leftmargin=*]

\item $g(\Sigma_w-I)$ is an upper bound on $v^\top (\Sigma_w - I) v$ for all $k$-sparse unit vectors $v \in \R^d$.
This way, a small objective value implies that the variance along all $k$-sparse directions is small, which guarantees that $\normtwok{\mu_w - \mus}$ is small.
\item $g(\Sigma_{w^\star}-I)$ is small for $w^\star$ (the uniform distribution on $G$).  This guarantees that a good $w$ exists.  
\item Triangle inequality on $g$.  This allows us to upper bound the objective value when we move $w$ toward $w^\star$ by the sum of $g(\cdot)$ of each term on the right-hand side:
\begin{align*}
&\Sigma_{(1-\eta)w + \eta w^\star} - I = (1-\eta)(\Sigma_w - I) + \eta(\Sigma_{w^\star} - I)
 + \eta(1-\eta)(\mu_w - \mu_{w^\star})(\mu_w - \mu_{w^\star})^\top \; .
\end{align*}
\item $g(u u^\top)$ is close to $g(v v^\top)$ where $v$ keeps only the $k$ largest entries of $u$.
We want to approximate $\mu$ in $\normtwok{\cdot}$ norm, so intuitively $g(\Sigma_w - I)$ should depend only on the largest $k$ entries of $(\mu_w - \mus)$.
\end{enumerate}

Our choice of $f(w) = g(\Sigma_w - I) = \fkknorm{\Sigma-I}$ is motivated by (and satisfies) all these properties.

\begin{lemma}
\label{lem:fkknorm-sparse-v}
Fix $A \in \R^{d \times d}$.
We have $\abs{v^\top A v} \le \fkknorm{A}$ for any $k$-sparse unit vector $v \in \R^d$.
\end{lemma}

\begin{lemma}
\label{lem:fkknorm-rank-one}
For any vector $v \in \R^d$, $\fkknorm{v v^\top} = \normtwok{v}^2$.
\end{lemma}

We now continue to present key technical lemmas for proving our main structural result (Theorem~\ref{thm:sparse-mean-struc}).
Lemma~\ref{lem:sample-cov-mixed-w} gives the weighted empirical covariance for a convex combination of two weight vectors.
\begin{lemma}
\label{lem:sample-cov-mixed-w}
Fix $n$ samples $X_1, \ldots, X_n \in \R^d$.
Let $\bar w, \hat w \in \R^n$ be two non-negative weight vectors with $\normone{\bar w} = \normone{\hat w} = 1$.
For any $\alpha, \beta \ge 0$ with $\alpha + \beta = 1$, letting $w = \alpha \bar w + \beta \hat w$, we have
\[
\Sigma_w = \alpha \Sigma_{\bar w} + \beta \Sigma_{\hat w} + \alpha \beta (\mu_{\bar w} - \mu_{\hat w})(\mu_{\bar w} - \mu_{\hat w})^\top \; .
\]
\end{lemma}
\begin{proof}
Because $w = \alpha \bar w + \beta \hat w$ and $\mu_w$ is linear in $w$, we have $\mu_w = \alpha \mu_{\bar w} + \beta \mu_{\hat w}$.
The lemma follows from the following calculations:
\begin{align*}
\Sigma_w = \sum_i w_i X_i X_i^\top - \mu_w \mu_w^\top
&= \sum_i \alpha \bar w_i X_i X_i^\top - \alpha \mu_{\bar w} \mu_{\bar w}^\top
  + \sum_i \beta \hat w_i X_i X_i^\top - \beta \mu_{\hat w} \mu_{\hat w}^\top \\
&\qquad + \alpha \mu_{\bar w} \mu_{\bar w}^\top + \beta \mu_{\hat w} \mu_{\hat w}^\top 
  - (\alpha \mu_{\bar w} + \beta \mu_{\hat w}) (\alpha \mu_{\bar w} + \beta \mu_{\hat w})^\top \\
&= \alpha \Sigma_{\bar w} + \beta \Sigma_{\hat w} + \alpha \beta (\mu_{\bar w} - \mu_{\hat w})(\mu_{\bar w} - \mu_{\hat w})^\top \; .
\end{align*}
The last step uses $\alpha - \alpha^2 = \beta - \beta^2 = \alpha \beta$ as $\alpha + \beta = 1$.
\end{proof}

Let $w^\star$ denote the uniform distribution on $G$, i.e., $w^\star_i = \frac{1}{(1-\eps)n}$ if $i \in G$ and $w^\star_i = 0$ otherwise.
By Lemma~\ref{lem:sample-cov-mixed-w} for any $w$, if we move toward $w^\star$, we have
\[
\Sigma_{(1-\eta)w + \eta w^\star} = (1-\eta)\Sigma_w + \eta\Sigma_{w^\star} + \eta(1-\eta)(\mu_w - \mu_{w^\star})(\mu_w - \mu_{w^\star})^\top \; .
\]

We will show that we can essentially ignore the last rank-one term using Lemma~\ref{lem:cross-term-sparse}. 

\begin{lemma}
\label{lem:cross-term-sparse}
Let $G^\star$ be a $(k,\eps,\delta)$-stable set of samples with respect to the ground-truth distribution with $0 < \eps \le \delta$.
Let $S$ be an $\eps$-corrupted version of $G^\star$.
Then, we have
\begin{align*}
&\fkknorm{(\mu_w - \mu_{w_\star})(\mu_w - \mu_{w_\star})^\top}
\le 4 \eps \left( \fkknorm{\Sigma_w - I} + O({\delta^2}/{\eps}) \right) \; .
\end{align*}
\end{lemma}

We are now ready to prove our main result (Theorem~\ref{thm:sparse-mean-struc}).

\begin{proof}[Proof of Theorem~\ref{thm:sparse-mean-struc}]
Fix any weight vector $w \in \Delta_{n,\eps}$.
We will show that if $w$ is a bad solution, then $f(w)$ decreases if $w$ moves toward $w^\star$, so $w$ cannot be a stationary point.

Let $c_1$ be the constant in $O(\cdot)$ in Lemma~\ref{lem:cross-term-sparse}.
By Lemma~\ref{lem:cross-term-sparse}, if $\normtwok{\mu_w - \mus} \ge c_2 \delta$ for a sufficiently large constant $c_2$, then $\fkknorm{\Sigma_w - I} \ge (\frac{c_2^2}{4} - c_1)\frac{\delta^2}{\eps} = \Omega(\frac{\delta^2}{\eps})$.

By Lemma~\ref{lem:sample-cov-mixed-w},
\begin{align*}
\Sigma_{(1-\eta)w + \eta w^\star} - I
= (1-\eta)(\Sigma_w - I) + \eta(\Sigma_{w^\star} - I)
+ \eta(1-\eta)(\mu_w - \mu_{w^\star})(\mu_w - \mu_{w^\star})^\top.
\end{align*}

Using the triangle inequality for $\fkknorm{\cdot}$, we have
\begin{align*}
\fkknorm{\Sigma_{(1-\eta)w + \eta w^\star}-I}
&\le (1-\eta)\fkknorm{\Sigma_w - I} \\ 
&\qquad + \eta \fkknorm{\Sigma_{w^\star} - I} + \eta(1-\eta) \fkknorm{(\mu_w - \mu_{w^\star})(\mu_w - \mu_{w^\star})^\top} \; .
\end{align*}

We know that $\fkknorm{\Sigma_{w^\star} - I} \le \frac{\delta^2}{\eps}$ by the stability condition in Definition~\ref{def:sparse-stability}.
Using Lemma~\ref{lem:cross-term-sparse} and that $\fkknorm{\Sigma_w - I} = \Omega({\delta^2}/{\eps})$, we can show that for all $0 < \eta < 1$,
\begin{align}
\begin{split}
\label{eqn:sp-mean-vii}
f((1-\eta)w + \eta w^\star)
&= \fkknorm{\Sigma_{(1-\eta)w + \eta w^\star}-I} \\
&\le (1-\eta)\fkknorm{\Sigma_w - I} + \tfrac{\eta \delta^2}{\eps} + 4\eps\eta \left( \fkknorm{\Sigma_w - I} + O(\tfrac{\delta^2}{\eps}) \right) \\
&\le (1-\eta+4\eps\eta) \fkknorm{\Sigma_w - I} + (4c_1+1)\tfrac{\eta \delta^2}{\eps} \\
&\le (1-\tfrac{\eta}{2}) \fkknorm{\Sigma_w - I}
= (1-\tfrac{\eta}{2}) f(w) \; .
\end{split}
\end{align}
The last step uses $(\frac{1}{2}-4\eps)\normtwo{\Sigma_w - I} \ge (4c_1 + 1)\frac{\delta^2}{\eps}$ which holds if $\eps \le 1/10$ and $c_2^2 \ge 164 \, c_1 + 40$.

It follows immediately that $w$ cannot be a stationary point of $f$.
Let $u = \frac{w^\star - w}{\normtwo{w^\star - w}}$ and $h = \eta \normtwo{w^\star - w}$.
We have $w + h u = (1-\eta)w + \eta w^\star \in \Delta_{n, \eps}$ because $\Delta_{n, \eps}$ is a convex set.
Since $\normtwo{w^\star - w} = O(\sqrt{\eps/n})$,
\begin{align*}
\textstyle u^\top \nabla f(w) &= \lim_{h \to 0} \tfrac{f(w + h u) - f(w)}{h} 
\le \lim_{\eta \to 0} \tfrac{-(\eta/2)f(w)}{\eta \normtwo{w^\star - w}}
\le -\tfrac{\Omega(\delta^2 / \eps)}{\normtwo{w^\star - w}} \le -\Omega(n^{1/2} \delta^2 \eps^{-3/2}) \; .
\end{align*}
By Definition~\ref{def:stationary}, we know $w$ cannot be a $\gamma$-stationary point of $f$ for some $\gamma = O(n^{1/2} \delta^2 \eps^{-3/2})$.
\end{proof}


\section{Robust Sparse PCA}
\label{sec:sparse-pca}

We consider a spiked covariance model for sparse PCA. In this model, there is a direction $v\in \R^d$ with at most $k$ nonzero entries. The good samples are drawn from a ground-truth distribution with covariance $\Sigma = I+\rho vv^\top$, where $\rho > 0$ is a parameter that intuitively measures the strength of the signal. We consider the more interesting case when $\rho \le 1$ (if $\rho$ is larger the problem becomes easier).

To solve the sparse PCA problem, we consider the following optimization problem, where $M_w = \sum_i w_i X_i X_i^\top$ and $\ftwoksnorm{A} = \max_{|Q| = 2k^2} \fnorm{A_Q}$:
\begin{equation}
\label{eqn:sparse-pca-obj}
\min_w \; f(w) = \ftwoksnorm{M_w - I} \; \textrm{ subject to } \; w \in \Delta_{n,\eps} .
\end{equation}

The objective function minimizes the Frobenius norm of the largest $2k^2$ entries of a reweighted second-moment matrix $M_w$.
Note that $f(w)$ is in fact convex in $w$, because the second-moment matrix $M_w$ is linear in $w$ and the $\ftwoksnorm{\cdot}$ norm is convex.

Let $R$ be the support of $v v^\top$.
Intuitively, the $k^2$ entries in $R$ could be large due to spiked covariance.
By minimizing the norm of the largest $2 k^2$ entries, we hope to make the entries outside of $R$ very small.
Once we optimize this objective function, it is not difficult to find a vector $u$ that is close to $v$.
Our algorithm is given in Algorithm~\ref{alg:sparsepca}.

\begin{algorithm}[h]
\caption{Robust sparse PCA.}\label{alg:sparsepca}
{\bfseries Input:} $k > 0$, $0 < \eps < \eps_0$, and an $\eps$-corrupted set of samples $(X_i)_{i=1}^n$ drawn from a distribution with covariance $I + \rho v v^\top$ for a $k$-sparse unit vector $v$. \\
{\bfseries Output:} a vector $u$ that is close to $v$.

\begin{algorithmic}[1]
\STATE Find a first-order stationary point $w \in \Delta_{n,\eps}$ of the objective $\min_w f(w) = \ftwoksnorm{M_w - I}$.
\STATE Let $A = M_w - I$.  Let $Q$ be the $k^2$ entries of $A$ with largest magnitude.
\STATE Return $u = $ the top eigenvector of $(A_Q + A_Q^\top)$.
\end{algorithmic}
\end{algorithm}


\begin{theorem}
\label{thm:sparse-pca-struc}
Let $0 < \rho \le 1$, $0 < \eps < \eps_0$, and $\delta > \eps$.
Let $G^\star$ be a set of $n$ samples that is $(k,\eps,\delta)$-stable (as in Definition~\ref{def:sparse-pca-stability}) w.r.t. a centered distribution with covariance $I + \rho vv^\top$ for an unknown $k$-sparse unit vector $v \in \R^d$.
Let $S = (X_i)_{i=1}^n$ be an $\eps$-corrupted version of $G^\star$.
Algorithm~\ref{alg:sparsepca} outputs a vector $u$ such that $\fnorm{uu^\top - vv^\top} = O(\sqrt{\delta/\rho})$.
\end{theorem}

Theorem~\ref{thm:intro-sparse-pca-struc} is an immediate corollary of Theorem~\ref{thm:sparse-pca-struc}.

We can apply Theorem~\ref{thm:sparse-pca-struc} to get an end-to-end result for subgaussian distributions.
Algorithm~\ref{alg:sparsepca} requires the stability conditions (Definition~\ref{def:sparse-pca-stability}) of the original good samples $G^\star$.
We show that these conditions are satisfied with a small number of samples.

\begin{lemma}
\label{lem:sparse-pca-stability}
Let $0 < \rho \le 1$ and $0 < \eps < \eps_0$.
Let $D$ be a centered subgaussian distribution with covariance $I + \rho vv^\top$ for a $k$-sparse unit vector $v \in \R^d$.
Let $G^\star$ be a set of $n = \Omega(k^2 \log d/\delta^2)$ samples drawn from $D$. Then then with probability at least $1-\exp(-\Omega(k^2\log d))$, $G^\star$ is $(k,\eps, \delta)$-stable (as in Definition~\ref{def:sparse-pca-stability}) w.r.t. $D$ for $\delta = O(\eps\log(1/\eps))$.
\end{lemma}

Combining Theorem~\ref{thm:sparse-pca-struc} and Lemma~\ref{lem:sparse-pca-stability}, given as input an $\eps$-corrupted set of $n = \tilde{\Omega}(k^2 \log d/\eps^2)$ samples drawn from a centered subgaussian distribution with covariance $I + \rho vv^\top$, Algorithm~\ref{alg:sparsepca} returns a vector $u$ with $\fnorm{uu^\top - vv^\top} = O(\sqrt{\eps\log(1/\eps)/\rho})$.

We defer the proofs of Lemma~\ref{lem:sparse-pca-stability} and Theorem~\ref{thm:sparse-pca-struc} to Appendix~\ref{apx:sparse-pca} and give an overview of the proof of Theorem~\ref{thm:sparse-pca-struc}.

\paragraph{Proof Sketch of Theorem~\ref{thm:sparse-pca-struc}.}
We can use the stability conditions to upper bound the optimal objective value: note that for $w^\star$ (uniform distribution on the remaining good samples), we must have $\ftwoksnorm{M_{w^\star} - (I + \rho v v^\top)} \le \delta$ by the stability conditions, therefore $\ftwoksnorm{M_{w^\star} - I} \le \ftwoksnorm{M_{w^\star} - (I + \rho v v^\top)} + \ftwoksnorm{\rho vv^\top} \le \rho + \delta$.
Because the objective function $f(w)$ is convex, any stationary point $w$ must be globally optimal and therefore satisfies $f(w) \le \rho + \delta$. 

Fix a stationary point $w$ and let $A = M_w - I$.
Let $R$ be the support of $vv^\top$ and let $Q$ be the set of $k^2$ largest entries of $A$.
The stability conditions implies for any $w$, the projection in the $v$ direction must be large (formally $v^\top A v \ge \rho - \delta$).
Because the objective function measures the norm of the largest $2 k^2$ entries of $A$ and it is not much larger than the norm of the largest $k^2$ entries, we can argue that $A_R$ and $A_Q$ are close, so $v^\top A_Q v \ge \rho - O(\delta)$.

Now for any stationary point $w$, $A_Q = (M_w - I)_Q$ is a matrix with Frobenius norm at most $\rho+\delta$ while $v^\top A_Q v \ge \rho - O(\delta)$. Together these imply that the norm of $A_Q$ in direction orthogonal to $vv^\top$ is at most $O(\sqrt{\rho\delta})$, and then by standard matrix perturbation bounds we know the top eigenvector of $(A_Q+A_Q^\top)$ is $O(\sqrt{\delta/\rho})$ close to $v$.

\section{Experiments}
\label{sec:experiment}
We perform an experimental evaluation of our robust sparse mean estimation algorithm on synthetic datasets
with a focus on statistical accuracy ($\ell_2$-distance between the output and the true sparse mean).
We evaluate our algorithm (which we term Sparse Gradient Descent, \texttt{Sparse GD}) on different noise models, 
and compare it to the following previous algorithms:
\begin{itemize}[leftmargin=*]
    \item \texttt{oracle}, which is told exactly which samples are inliers, and outputs their empirical mean,
    \item the robust sparse mean estimation algorithm \texttt{RME\_sp} from \cite{DKKPS19-sparse},
    \item \texttt{NP}, the naive pruning algorithm that removes samples too far from the empirical median, and outputs the mean of the remaining samples,
    \item \texttt{RANSAC}, which randomly selects half of the points and computes their mean. One solution is preferred to another if it has more points in a ball of radius $O(\sqrt{d})$ around it.
\end{itemize}
For algorithms that output non-sparse vectors, we take the largest $k$ entries before measuring the $\ell_2$ distance to the true mean.
We evaluate the algorithms on various noise models 
\begin{itemize}[leftmargin=*]
    \item \textbf{Linear-hiding noise.} The inliers are drawn from $\mathcal N(0, I)$. Let $S$ be a size $k$ set. Then, half the outliers are drawn from $\mathcal N(1_S, I)$ and the other half are drawn from $\mathcal N(0, 2 I - I_S)$.
    \item \textbf{Tail-flipping noise.} This noise model picks a $k$-sparse direction $v$ and replaces the $\epsilon$ fraction of points farthest in the $-v$ direction with points in the $+v$ direction.
    \item \textbf{Constant-bias noise.} This model adds a constant to every coordinate of the outlier points. In Figure \ref{fig:constant_bias}, we add $2$ to every coordinate of every outlier point.
\end{itemize}

We ran our experiments on a computer with a 1.6 GHz Intel Core i5 processor and 8 GB RAM.
We built on the codebase of~\cite{DKKPS19-sparse}~\footnote{Available at: \url{https://github.com/sushrutk/robust_sparse_mean_estimation}, MIT license} and implemented our new algorithm for the experiments.
For each pair of algorithm and noise model, we repeat the same experiment $10$ times and plot the median value of the measurements. We shade the interquartile region around the reported points in the figure as confidence intervals.

\begin{figure*}[ht]
\centering
    \subfloat[
    Fix the sparsity $k$ and change the number of samples $n$.]{\makebox[0.52\linewidth][c]{\includegraphics[height=.27\linewidth]{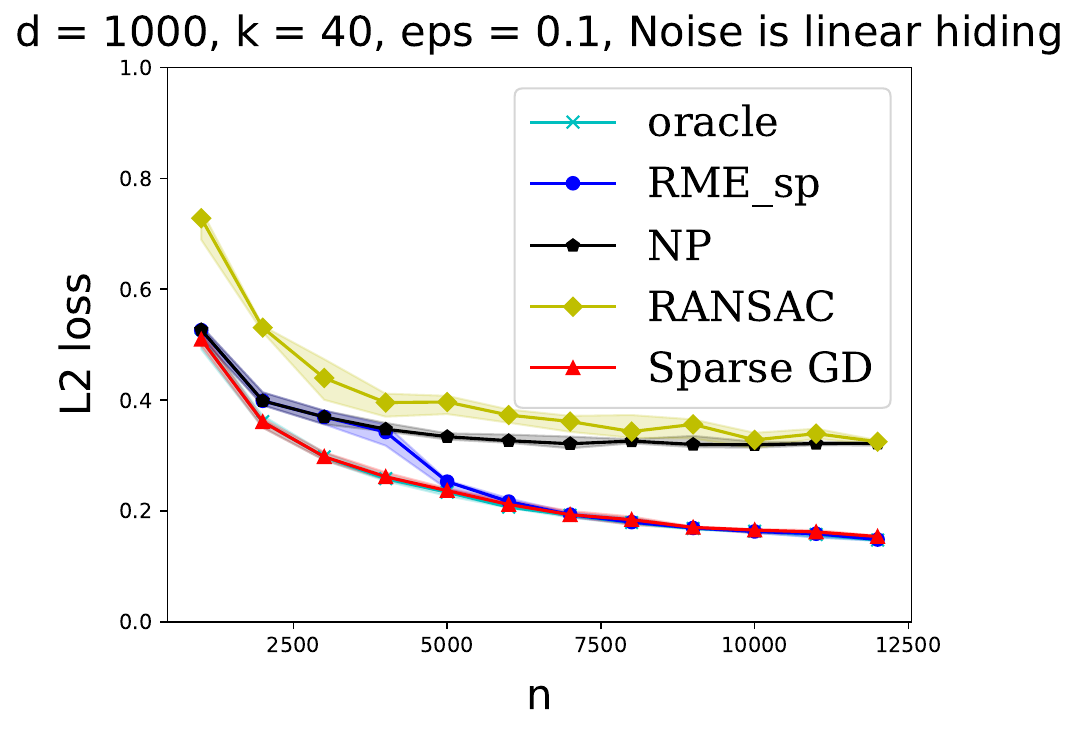}}}
    \hfill
    \subfloat[
    Fix $n$ and change the sparsity $k$.] {\makebox[0.46\linewidth][c]{\includegraphics[height=.27\linewidth]{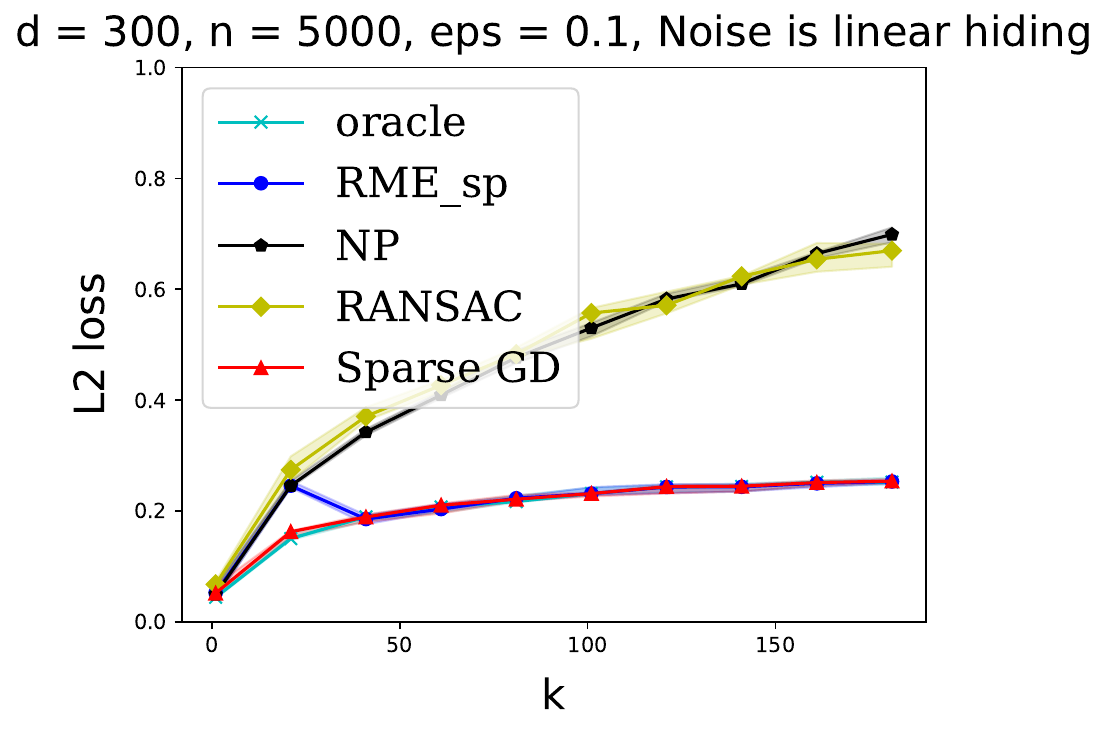}}}
\caption{The performance of various algorithms under linear-hiding noise. 
Notably, when the number of samples $n$ or the sparsity $k$ is small, our algorithm \texttt{Sparse GD} outperforms \texttt{RME\_sp}.}
\label{fig:linear_hiding}
\end{figure*}

\begin{figure*}[ht]
\centering
    \subfloat[
    Fix the sparsity $k$ and change the number of samples $n$.]{\makebox[0.52\linewidth][c]{\includegraphics[height=.27\linewidth]{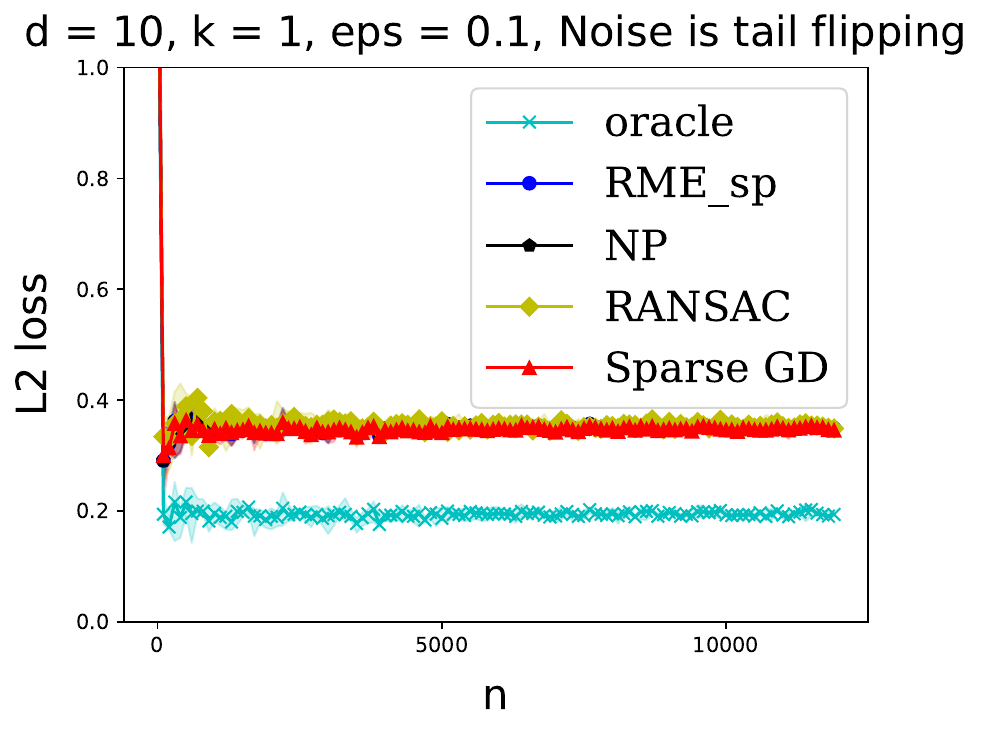}}}
    \hfill
    \subfloat[
    Fix $n, k$ and change the fraction of corruption $\eps$.]{\makebox[0.46\linewidth][c]{\includegraphics[height=.27\linewidth]{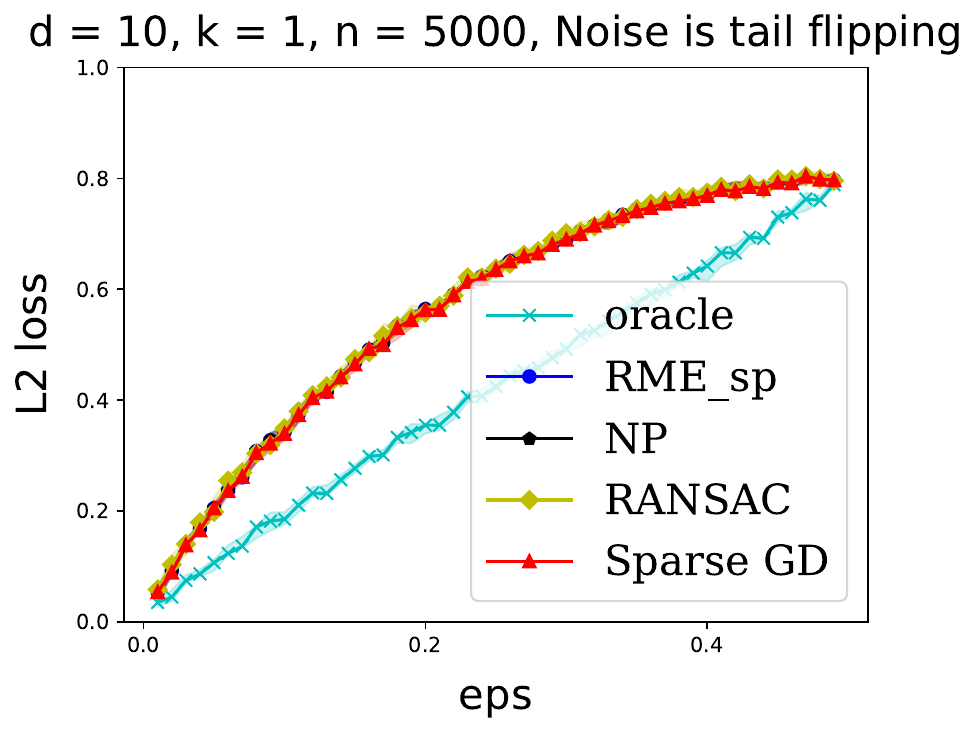}}}
\caption{The performance of various algorithms in the tail-flipping noise model. }
\label{fig:tail_flipping}
\end{figure*}

\begin{figure*}[ht]
\centering
    \subfloat[
    Fix the sparsity $k$ and change the number of samples $n$.]{\makebox[0.52\linewidth][c]{\includegraphics[height=.27\linewidth]{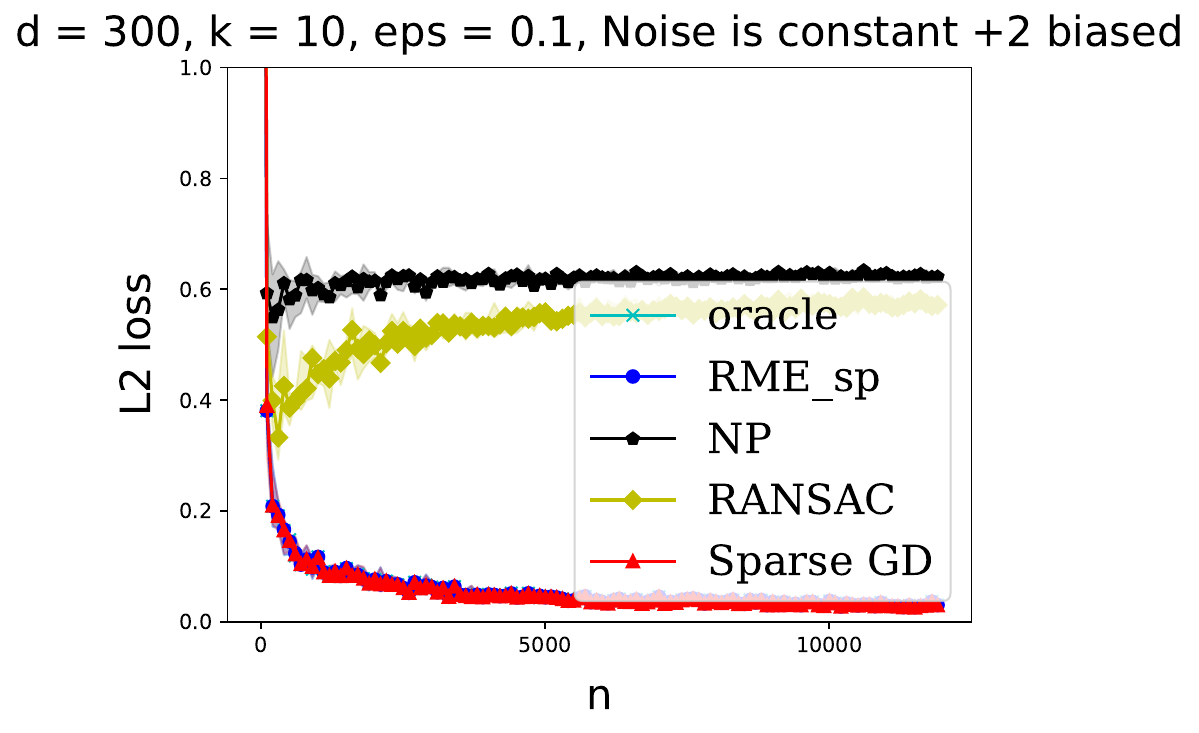}}}
    \hfill
    \subfloat[
    Fix $n$ and change the sparsity $k$.]{\makebox[0.46\linewidth][c]{\includegraphics[height=.27\linewidth]{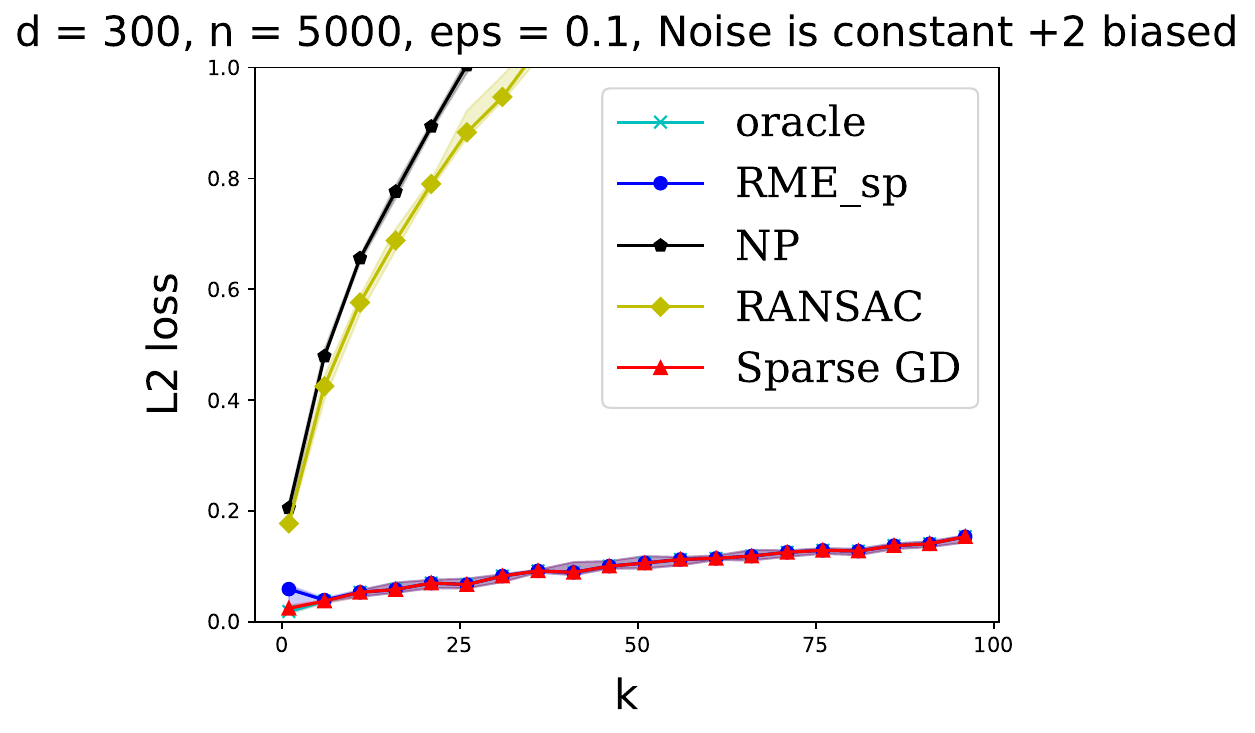}}}
\caption{The performance of various algorithms in the constant-bias noise model. }
\label{fig:constant_bias}
\end{figure*}

Our experimental results are summarized in Figures~\ref{fig:linear_hiding}, \ref{fig:tail_flipping} and \ref{fig:constant_bias}. For the linear-hiding and constant-bias noise models, we run two experiments: 1) fix the sparsity $k$ and change the number of samples $n$, and 2) fix $n$ and change $k$. For the tail-flipping noise model, we run two experiments: 1) fix the sparsity $k$ and change the number of samples $n$, and 2) fix $k$ and $n$ and change the fraction of corruption $\eps$.

In terms of statistical accuracy, our algorithm (\texttt{Sparse GD}), outperforms the filter-based \texttt{RME\_sp} algorithm~\cite{DKKPS19-sparse} in the linear-hiding noise model when the number of samples or the sparsity is small, as shown in Figure \ref{fig:linear_hiding}.
Our algorithm matches the performance of \texttt{RME\_sp} under the tail flipping and constant-bias noise models, as shown in Figures \ref{fig:tail_flipping} and \ref{fig:constant_bias}.

Matching our theoretical results, our \texttt{Sparse GD} algorithm has accuracy $O(\eps \sqrt{\log(1/\eps)})$ and is within a constant factor of the $\Omega(\eps \sqrt{\log(1/\eps)})$ worst-case performance of \texttt{oracle}. In contrast, the naive algorithms \texttt{NP} and \texttt{RANSAC} both incur error that scales as $\eps \sqrt{k}$.
The tail-flipping noise (Figure~\ref{fig:tail_flipping}) illustrates that $\Omega(\epsilon \sqrt{\log(1/\epsilon)})$ error can occur no matter which algorithm is used (including \texttt{oracle}), because $\eps$-fraction of the original good samples was removed.


\bibliographystyle{alpha}
\bibliography{allrefs}

\appendix


\section{Omitted Proofs in Section~\ref{sec:sparse-mean}}
\label{apx:sparse-mean}
In this section, we provide the proofs omitted from Section~\ref{sec:sparse-mean}.

We start with the key technical lemmas we used for our structural result for robust sparse mean estimation.
The sample complexity and the stability conditions for sparse mean will be proved in Appendix~\ref{apx:sparse-mean-stability}.

We will restate the lemmas before proving them.

\smallskip
{\noindent \bf Lemma~\ref{lem:thres-normtwok}.~} {\em
Fix two vectors $x, y$ with $\normzero{x} \le k$ and $\normtwok{x-y} \le \delta$.
Let $z$ be a vector that keeps the $k$ entries of $y$ with largest absolute values and sets the rest to $0$.
We have $\normtwo{x-z} \le \sqrt{5} \delta$.
}
\begin{proof}
Without loss of generality we assume $\normzero{x} = k$.
We partition the coordinates into three disjoint sets based on the sparsity of $x$ and $z$.
Let $A = \{i : x_i \neq 0 \text{ and } z_i = 0 \}$,
$B = \{i : x_i = 0 \text{ and } z_i \neq 0 \}$, and
$C = \{i : x_i \neq 0 \text{ and } z_i \neq 0 \}$.

We know that $\abs{y_i} \le \abs{y_j}$ for all $i \in A$ and $j \in B$ because the $k$ largest entries of $\abs{y}$ are in $B \cup C$.
Since $|A \cup C| = |B \cup C| = k$, we have $|A| = |B|$ and therefore $\normtwo{y_A} \le \normtwo{y_B}$.

By the definition of $z$, we have $z_i = 0$ for $i \in A$ and $z_i = y_i$ for all $i \in B \cup C$.
We have
\begin{align*}
\normtwo{x-z}^2
&= \normtwo{x_A}^2 + \normtwo{y_B}^2 + \normtwo{x_C - y_C}^2 \tag*{($y_A = 0$ and $x_B = 0$)} \\
&\le 2(\normtwo{x_A - y_A}^2 + \normtwo{y_A}^2) + \normtwo{y_B}^2 + \normtwo{x_C - y_C}^2 \tag*{(triangle inequality)} \\
&\le 2 \normtwo{x_A - y_A}^2 + 3 \normtwo{y_B}^2 + \normtwo{x_C - y_C}^2 \tag*{($\normtwo{y_A} \le \normtwo{y_B}$)} \\
&\le 2\normtwo{x_{A\cup B\cup C} - y_{A\cup B\cup C}}^2 + \normtwo{y_B}^2 \tag*{($A$, $B$, $C$ are disjoint)} \\
&\le 5 \normtwok{x-y}^2 \; .
\end{align*}
The last inequality uses $|A \cup B \cup C| \le 2k$ and $\left\| x - y \right\|_{2,2k}^2 \le 2\normtwok{x-y}^2$.
\end{proof}

\smallskip
{\noindent \bf Lemma~\ref{lem:fkknorm-sparse-v}.~} {\em
Fix $A \in \R^{d \times d}$.
We have $\abs{v^\top A v} \le \fkknorm{A}$ for any $k$-sparse unit vector $v \in \R^d$.
}
\begin{proof}
Without loss of generality, we assume $\normzero{v} = k$.

Let $R \subseteq ([d] \times [d])$ be the support of $vv^\top$.
We have
\[
v^\top A v \le \normtwo{A_R} \le \fnorm{A_R} \le \fkknorm{A} \;
\]
The last inequality is because $\fkknorm{A}$ chooses a set of $k^2$ entries to maximize the $\ell_2$-norm of these entries, subject to choosing these entries from $k$ rows with $k$ entries on each row, and $R$ is a feasible way to choose such $k^2$ entries.
\end{proof}

\smallskip
{\noindent \bf Lemma~\ref{lem:fkknorm-rank-one}.~} {\em
For any $v \in \R^d$, $\fkknorm{v v^\top} = \normtwok{v}^2$.
}
\begin{proof}
Without loss of generality, we can assume $v$ is non-negative because the norms on both sides are independent of signs.  Moreover, we can assume w.l.o.g. that the coordinates of $v$ are sorted from large to small, i.e., $v_1 \ge v_2 \ge \ldots \ge v_d \ge 0$.

The rows are multiples of each other, so the $k$ rows with largest $(\ell_2,k)$-norms are the first $k$ rows, and the $(\ell_2,k)$-norm of each row is just the $\ell_2$-norm of the first $k$ entries.
Therefore, we have $\fkknorm{v v^\top}^2 = \sum_{i=1}^k \sum_{j=1}^k (v_i v_j)^2 = \left(\sum_{i=1}^k v_i^2\right)^2 = \normtwok{v}^4$ as claimed.
\end{proof}


\smallskip
{\noindent \bf Lemma~\ref{lem:cross-term-sparse}.~} {\em
Let $G^\star$ be a $(k,\eps,\delta)$-stable set of samples with respect to the ground-truth distribution with $0 < \eps \le \delta$.
Let $S$ be an $\eps$-corrupted version of $G^\star$.
Then, we have
}
\[
\fkknorm{(\mu_w - \mu_{w_\star})(\mu_w - \mu_{w_\star})^\top} \le 4 \eps \left( \fkknorm{\Sigma_w - I} + O(\tfrac{\delta^2}{\eps}) \right) \; .
\]
\begin{proof}
Recall that $S = G \cup B$ where $G$ is the set of (remaining) good samples and $B$ is the set of corrupted samples.
Let $\alpha = \normone{w_G}$ and $\beta = \normone{w_B}$.
Let $\bar{w} = w_G / \alpha$ and $\hat{w} = w_B / \beta$ denote the normalized version of $w_G$ and $w_B$.

We can write $w = \alpha \bar{w} + \beta \hat{w}$, by Lemma~\ref{lem:sample-cov-mixed-w}, we know that
\begin{align}
\label{eqn:sp-mean-i}
\Sigma_w &= \alpha \Sigma_{\bar{w}} + \beta \Sigma_{\hat{w}} + \alpha \beta (\mu_{\bar{w}} - \mu_{\hat{w}})(\mu_{\bar{w}} - \mu_{\hat{w}})^\top \; .
\end{align}

Since $\beta \le \norminf{w} |B| \le \frac{\eps}{1-\eps}$, we have $\norminf{\bar w} = \frac{\norminf{w_G}}{\alpha} \le \frac{1}{(1-2\eps)n}$. 
Because $G^\star$ is $(k,\eps,\delta)$-stable and $\bar w \in \Delta_{n,2\eps}$ can be viewed as weights on $G^\star$, by the stability condition in Definition~\ref{def:sparse-stability},
\begin{align}
\label{eqn:sp-mean-ii}
\fkknorm{\Sigma_{\bar w} - I} \le \tfrac{\delta^2}{\eps} \; .
\end{align}

Using Lemma~\ref{lem:fkknorm-sparse-v}, Equations~\eqref{eqn:sp-mean-i}~and~\eqref{eqn:sp-mean-ii}, and that $\Sigma_{\hat w} \succeq 0$, for any $k$-sparse unit vector $v \in \R^d$, 
\begin{align}
\label{eqn:sp-mean-iii}
\begin{split}
\fkknorm{\Sigma_w - I}
\ge v^\top (\Sigma_w - I) v
&= \alpha v^\top \Sigma_{\bar{w}} v + \beta v^\top \Sigma_{\hat{w}} v + \alpha \beta \left( (\mu_{\bar w} - \mu_{\hat w})^\top v \right)^2 - 1 \\
&\ge \alpha\left(1+v^\top(\Sigma_{\bar w} - I)v\right) + \alpha \beta \left( (\mu_{\bar w} - \mu_{\hat w})^\top v \right)^2 - 1 \\
&\ge \alpha\left(1-\tfrac{\delta^2}{\eps}\right) - 1 + \alpha \beta \left( (\mu_{\bar w} - \mu_{\hat w})^\top v \right)^2 \; .
\end{split}
\end{align}

We know $\alpha (1-\delta^2/\eps)$ is close to $1$, so we are essentially left with only the last term on the right-hand side.
We will relate this term to $\normtwok{\mu_w - \mu_{w_\star}}^2$, which is what appears in the lemma statement.

Recall that $\alpha + \beta = 1$ and $w = \alpha \bar{w} + \beta \hat{w}$, and thus
\begin{align}
\label{eqn:sp-mean-iv}
\beta(\mu_{\hat w} - \mu_{\bar w}) &= \beta \mu_{\hat w} + \alpha \mu_{\bar w} - \mu_{\bar w} = \mu_w - \mu_{\bar w} = (\mu_w - \mu_{w^\star}) + (\mu_{w^\star} - \mu_{\bar w}) \; .
\end{align}

Since $\bar w, w^\star \in \Delta_{n, 2\eps}$ put weights only on $G$, it follows from the stability conditions (Definition~\ref{def:sparse-stability})
\begin{align}
\label{eqn:sp-mean-v}
\abs{\left(\mu_{w^\star} - \mu_{\bar w} \right)^\top v}
\le \normtwo{\mu_{w^\star} - \mu_{\bar w}}
\le \normtwo{\mu_{w^\star} - \mus} + \normtwo{\mus - \mu_{\bar w}}
\le 2 \delta \; .
\end{align}

Let $u \in \R^d$ be a vector that keeps the $k$ entries of $(\mu_w - \mu_{w^\star})$ with largest magnitude.
We choose $v = \frac{u}{\normtwo{u}}$.
Notice that $(\mu_w - \mu_{w^\star})^\top v = \normtwok{\mu_w - \mu_{w^\star}}$.
From Equations~\eqref{eqn:sp-mean-iv}~and~\eqref{eqn:sp-mean-v}, we have
\begin{align}
\label{eqn:sp-mean-vi}
\begin{split}
\! \left( \beta \cdot (\mu_{\hat w} - \mu_{\bar w})^\top v \right)^2
&= \left((\mu_w - \mu_{w^\star})^\top v + (\mu_{w^\star} - \mu_{\bar w})^\top v \right)^2 \\
&\ge \frac{\left( (\mu_{w} - \mu_{w_\star})^\top v \right)^2}{2} - \left( (\mu_{\bar w} - \mu_{w^\star})^\top v \right)^2
\ge \frac{\normtwok{\mu_{w} - \mu_{w_\star}}^2}{2} - 4 \delta^2 \; .
\end{split}
\end{align}
The first inequality in Equation~\eqref{eqn:sp-mean-vi} uses the fact that $(x+y)^2 \ge \frac{x^2}{2} - y^2$ for any $x, y \in \R$.

Putting Equations~\eqref{eqn:sp-mean-iii}~and~\eqref{eqn:sp-mean-vi} together for our choice of $v$, we have
\begin{align*}
\fkknorm{\Sigma_w - I}
&\ge \alpha \left(1-\tfrac{\delta^2}{\eps}\right) - 1 + \frac{\alpha}{\beta} \left( \beta \cdot(\mu_{\bar w} - \mu_{\hat w})^\top v \right)^2 \\
&\ge \frac{1-2\eps}{1-\eps} \left(1-\tfrac{\delta^2}{\eps}\right) - 1
  + \frac{1-2\eps}{\eps}\left(\frac{\normtwok{\mu_{w} - \mu_{w_\star}}^2}{2} - 4 \delta^2\right) \\
&\ge \frac{1}{4\eps} \normtwok{\mu_{w} - \mu_{w_\star}}^2 - O\left(\tfrac{\delta^2}{\eps}\right) \; .
\end{align*}
The lemma follows since $\normtwok{\mu_{w} - \mu_{w_\star}}^2 = \fkknorm{(\mu_w - \mu_{w_\star})(\mu_w - \mu_{w_\star})^\top}$ by Lemma~\ref{lem:fkknorm-rank-one}. 
\end{proof}


\subsection{Stability Conditions for Robust Sparse Mean} \label{apx:sparse-mean-stability}
In this section, we prove Lemma~\ref{lem:sparse-mean-stability}, which states that stability conditions (Definition~\ref{def:sparse-stability}) needed for our robust sparse mean algorithm is satisfied with a small number of samples (Lemma 3.3).

A version of the $\normtwok{\cdot}$ part of of Lemma~\ref{lem:sparse-mean-stability} was known in prior works (e.g., \cite{BDLS17}).  In this paper, we define the stability conditions with weights and we include the proof for the $\normtwok{\cdot}$ part to be self-contained.

\smallskip
{\noindent \bf Lemma~\ref{lem:sparse-mean-stability}.~} {\em
Fix $k > 0$ and $0 < \eps < \eps_0$. Let $G^\star$ be a set of $n$ samples that is generated according to a subgaussian distribution with mean $\mu$ and covariance $I$, if $n = \Omega(k^2\log d/\delta^2)$, then with probability at least $1-\exp(-\Omega(k^2\log d))$, $G^\star$ is $(k,\eps, \delta)$-stable (as in Definition~\ref{def:sparse-stability}) for $\delta = O(\eps\log(1/\eps))$.
}
\begin{proof}
Recall that the stability conditions in Definition~\ref{def:sparse-stability} state that for all $w \in \Delta_{n,2\eps}$,
\begin{align*}
\normtwok{\mu_w - \mu} \le \delta
\quad \text{ and } \quad
\fkknorm{\Sigma_w - I} \le \delta^2/\eps,
\end{align*}
Without loss of generality, we assume that the samples $G^\star = (X_i)_{i=1}^n$ are drawn from a subgaussian distribution with mean $\mu = 0$ and identity covariance matrix.
This is because shifting the samples by $\mu$ does not change the lemma statement.

For ease of presentation, we will upper bound the norms with $O(\delta)$ and $O(\delta^2/\eps)$ and prove the lemma for $\Delta_{n,\eps}$ instead of $\Delta_{n, 2\eps}$.
This is sufficient because the constants in $O(\cdot)$ and the constants due to changing $\eps$ to $2\eps$ can be put into $\delta = O(\eps\log(1/\eps))$.

\begin{enumerate}[leftmargin=0.25in]
\item[\em(i)]
First we prove $\normtwok{\mu_w} \le O(\delta)$ with probability at least $1 - \exp(k \log d -\Omega(n \delta^2))$.

Due to convexity of $\normtwok{\cdot}$, it is sufficient to upper bound $\normtwok{\mu_w}$ for all vertices of $\Delta_{n,\eps}$.
In other words, we need to show that
\[
\normtwok{\frac{1}{(1-\eps)n}\sum_{i \in G^\star \setminus L} X_i} = O(\delta) \; \text{ for every $L \subseteq G^\star$ with $|L| = \eps n$.}
\]

Ignoring the constants and by the triangle inequality, it suffices to show
\begin{align}
\label{eqn:mean-concentration-twok}
\normtwok{\frac{1}{n} \sum_{i \in G^\star} X_i} = O(\delta) \quad \text{ and } \quad
\normtwok{\frac{1}{n} \sum_{i \in L} X_i} = O(\delta) \text{ for all } |L|=\eps n \; .
\end{align}

By the definition of $\normtwok{\cdot}$, we need to show that for all $k$-sparse unit vector $v$,
\begin{align}
\label{eqn:mean-concentration-fixed-v}
& v^\top \left(\frac{1}{n} \sum_{i \in G^\star} X_i\right) = O(\delta) \quad \text{ and } \quad
v^\top \left(\frac{1}{n} \sum_{i \in L} X_i\right) = O(\delta) \text{ for all } |L|=\eps n \; .
\end{align}
We first prove Equation~\eqref{eqn:mean-concentration-fixed-v} for a fixed $v$ and then a take union bound over a net of $k$-sparse vectors to prove Equation~\eqref{eqn:mean-concentration-twok}.

Fix a unit vector $v \in \R^d$.

\begin{enumerate}
\item[\em(i.a)]
For $G^\star$, by the definition of subgaussian distributions and the Chernoff bound,
\[
\Pr\left[\left(\frac{1}{n} \sum_{i=1}^n v^\top X_i\right) > \delta \right] \le \exp(-\Omega(n \delta^2)) \; .
\]
\item[\em(i.b)]
For $L$, in the worst case, $L$ contains the samples with the largest $(v^\top X_i)$.
We will show that very few $(v^\top X_i)$ can be large.
Let
\[
h_r(z) = \begin{cases}
0 & z \le r \; , \\
z & z > r \; .
\end{cases}
\]
We have that, for every $L$ with $|L| = \eps n$,
\[
\frac{1}{n} \sum_{i \in L} v^\top X_i
\le \frac{1}{n} \sum_{i \in L} r + \frac{1}{n} \sum_{i \in L} h_r(v^\top X_i)
\le \eps r + \frac{1}{n} \sum_{i=1}^n h_r (v^\top X_i) \; .
\]
We set $r = 2 \sqrt{\ln(1/\eps)} > 1$.  The first term is $\eps r = O(\eps\sqrt{\log(1/\eps)}) = O(\delta)$, so we can focus on the second term.
Note that $h_r(v^\top X)$ is not bounded, but we can bound it using Chernoff-bound like arguments.
\begin{align*}
\expect{X}{\exp\left(h_r(v^\top X)/4\right)} 
&= \Pr[v^\top X \le r] + \int_{r}^\infty \frac{1}{\sqrt{2\pi}}\exp(-x^2/2) \exp(x/4) \; dx \\
&\le 1 + \int_{r}^\infty \frac{1}{\sqrt{2\pi}}\exp(-x^2/4) \; dx \\
&\le 1 + \exp(-r^2/4) \le \exp(\eps) \; .
\end{align*}

Because the $X_i$'s are independent, we have
\[
\expect{}{\exp\left(\frac{1}{4}\sum_{i=1}^n h_r(v^\top X_i)\right)} = \expect{}{\prod_{i=1}^n \exp\left(h_r(v^\top X_i)\right)} \le \exp(\eps n) \; .
\]
By Markov's inequality, $\frac{1}{4} \sum_{i=1}^n h_r(v^\top X_i) > 2 \delta n$ with probability at most $\exp((\eps - 2\delta) n) \le \exp(-n \delta)$.
\end{enumerate}

Therefore, Equation~\eqref{eqn:mean-concentration-fixed-v} hold for a fixed $v \in \R^d$ with probability at least $1 - \exp(-\Omega(n \delta^2))$.

We conclude Part {\em (i)} via a union bound over a net of $k$-sparse vectors.
Fix a sparsity pattern $R \subseteq [d]$ with $|R| = k$.
There exists a net $\CC_R$ of $2^{O(k)}$ unit vectors such that for any $y \in \R^d$, there is a vector $v \in \CC_R$ that aligns well with $y_R$ as in $v^\top y_R \ge (1/2) \normtwo{y_R}$.
Consequently, for any $y \in \R^d$,
\[
\normtwok{y} = \max_{|R|=k} \normtwo{y_R} \le 2 \max_R \max_{v \in \CC_R} v^\top y \; .
\]
Taking a union bound over $\binom{d}{k}$ sparsity patterns $R$ and every $v \in \CC_R$, we know that Equation~\eqref{eqn:mean-concentration-twok} holds with probability at least $1 - \exp(O(k \log d) -\Omega(n \delta^2))$, which then immediately implies $\normtwok{\mu_w} \le O(\delta)$.

\item[\em(ii)] Next we prove $\fkknorm{\Sigma_w - I} \le O(\delta^2/\eps)$.
The proof structure is similar to Part {\em (i)}.
The main difference is that we will need concentration inequality and tail bounds for the second moment, and for $\fkknorm{\cdot}$, we will consider a fixed matrix $U$ with $\fnorm{U} = 1$ (rather than a unit vector) and then union bound over all $k^2$-sparse matrices.

We first argue that it is sufficient to prove
\[
\fkknorm{M_w - I} = O(\delta^2/\eps) \; \text{ where } \; M_w = \sum_w w_i X_i X_i^\top \; .
\]
This is because
\begin{align*}
\fkknorm{\Sigma_w - I}
&= \fkknorm{M_w - \mu_w \mu_w^\top - I} \\
&\le \fkknorm{M_w - I} + \fkknorm{\mu_w \mu_w^\top} \\
&= \fkknorm{M_w - I} + \normtwok{\mu_w}^2 \tag*{(Lemma~\ref{lem:fkknorm-rank-one})}\\
&\le \fkknorm{M_w - I} + O(\delta^2) \tag*{($\normtwok{\mu_w} \le O(\delta)$ from Part {\em (i)})}
\end{align*}
Moreover, since $\fkknorm{\cdot}$ is convex and $(M_w - I)$ is linear in $w$, it is sufficient to consider all $w$ that is a vertex of $\Delta_{n,\eps}$.
In other words, we need to show for every $|L| = \eps n$,
\[
\fkknorm{\frac{1}{(1-\eps)n}\sum_{i \in G^\star \setminus L} X_i X_i^\top - I} = O(\delta^2/\eps) \; .
\]
Notice that
\[
\frac{1}{(1-\eps)n}\sum_{i \in G^\star \setminus L} X_i X_i^\top - I = \frac{1}{1-\eps} \left( \frac{1}{n} \sum_{i \in G^\star} \left(X_i X_i^\top - I\right) - \frac{1}{n} \sum_{i \in L} \left(X_i X_i^\top - I\right) \right) \; .
\]
Ignoring the constants and by the triangle inequality, it suffices to show
\begin{align}
\begin{split}
\label{eqn:cov-concentration-fkk}
\fkknorm{\frac{1}{n} \sum_{i \in G^\star} \left(X_i X_i^\top - I\right)} &= O(\delta^2/\eps) \; \text{ and } \\
\fkknorm{\frac{1}{n} \sum_{i \in L} \left(X_i X_i^\top - I\right)} &= O(\delta^2/\eps) \; \text{ for all } \; |L|=\eps n \; .
\end{split}
\end{align}

Because $\fkknorm{A} \le \fksnorm{A} = \max_{\normzero{U} \le k^2, \fnorm{U} = 1} (U \bullet A)$, it is sufficient to show that for all $k^2$-sparse matrix $U \in \R^{d\times d}$ with $\fnorm{U} = 1$,
\begin{align}
\begin{split}
\label{eqn:cov-concentration-fixed-u}
U \bullet \left(\frac{1}{n} \sum_{i \in G^\star} \left(X_i X_i^\top - I\right)\right) &= O(\delta^2/\eps) \; \text{ and } \\
U \bullet \left(\frac{1}{n} \sum_{i \in L} \left(X_i X_i^\top - I\right)\right) &= O(\delta^2/\eps) \; \text{ for all } \; |L|=\eps n \; .
\end{split}
\end{align}
We first prove Equation~\eqref{eqn:cov-concentration-fixed-u} for a fixed $U$ and then a take union bound over a net of $k^2$-sparse matrices to prove Equation~\eqref{eqn:cov-concentration-fkk}.

Fix a matrix $U \in \R^{d \times d}$ with $\fnorm{U} = 1$.

\begin{enumerate}
\item[\em (ii.a)]
Recall that $G^\star = (X_i)_{i=1}^n$ are drawn independently from a centered subgaussian distribution with covariance $\Sigma = I$.
Note that $\expect{X}{U \bullet (X X^\top - I)} = 0$.
By the Hanson-Wright inequality, for any $\fnorm{U} = 1$ and $0 < t \le 1$, we have
\[
\Pr\left[U \bullet \left(\frac{1}{n} \sum_{i \in G^\star} \left(X_i X_i^\top - I\right)\right) > t \right] \le \exp(-\Omega(n t^2)) \; .
\]
Therefore, $U \bullet \left(\frac{1}{n} \sum_{i \in G^\star} \left(X_i X_i^\top - I\right)\right) = O(\delta^2/\eps)$ holds with probability at least $1-\exp(-\Omega(n \delta^4 / \eps^2)) \ge 1-\exp(-\Omega(n \delta^2))$.

\item[\em (ii.b)]
For $L$, we will show that very few $U \bullet (X_i X_i^\top - I)$ can be large.
Recall that
\[
h_r(z) = \begin{cases}
0 & z \le r \; , \\
z & z > r \; .
\end{cases}
\]
For every $L$ with $|L| = \eps n$, we have
\begin{align}
\begin{split}
\label{eqn:cov-concentration-trunc}
\frac{1}{n} \sum_{i \in L} U \bullet (X_i X_i^\top - I)
&\le \frac{1}{n} \sum_{i \in L} r + \frac{1}{n} \sum_{i \in L} h_r(U \bullet (X_i X_i^\top - I)) \\
&\le \eps r + \frac{1}{n} \sum_{i=1}^n h_r (U \bullet (X_i X_i^\top - I)) \; .
\end{split}
\end{align}
We set $r = \delta^2/\eps^2$.

The first term is $\eps r = O(\delta^2/\eps)$, so we focus on the second term.
For the second term, let $c$ be a sufficiently small constant and consider
$\expect{}{\exp(c \cdot \sum_{i=1}^n h_r(U \bullet (X_i X_i^\top - I))}$.

Notice that by hypercontractivity, $U \bullet (X X^\top - I)$ has exponential tails (see, e.g.,~\cite{vershynin2018high}).
If $\delta$ is a sufficiently large multiple of $\eps\sqrt{\ln(1/\eps)}$, then $r$ will be a sufficiently large multiple of $\ln(1/\eps)$, and it will be the case that $h_r(U \bullet (X X^\top - I)) = 0$ except with probability at most $\eps^3$.
Observe that $\exp(c \cdot h_r(U \bullet (X X^\top - I))) > z$ iff $(U \bullet (X X^\top - I)) > \max(r, \ln(z)/c)$, which by the exponential tails (when $c$ is small enough) happens with probability at most $\min(\eps^3, 1/z^3)$.

We are now ready to bound the second term in Equation~\eqref{eqn:cov-concentration-trunc}.
\begin{align*}
&\expect{}{\exp(c \cdot h_r(U \bullet (X X^\top - I)))} \\ 
&= 1 + \int_{1}^{\infty} \Pr\left[\exp\left(c \cdot h_r(U \bullet (X X^\top - I))\right) > z\right] dz \\
&\le 1 + \int_{r}^{1/\eps} \eps^3 \; dz + \int_{1/\eps}^\infty 1/z^3 \; dz \\
&\le 1 + O(\eps^2) \le \exp(O(\eps^2)) \; .
\end{align*}

Because the $X_i$'s are independent, we have
\[
\expect{}{\exp\left(c \cdot \sum_{i=1}^n h_r(U \bullet (X_i X_i^\top - I))\right)} = \exp(O(n \eps^2)) \; .
\]
By Markov's inequality, $c \cdot \sum_{i=1}^n h_r(U \bullet (X X^\top - I)) > 2 (\delta^2/\eps) n$ with probability at most $\exp(n(O(\eps^2) - 2\delta^2/\eps)) \le \exp(-n \delta^2/\eps) \le \exp(-\Omega(n \delta^2))$.
\end{enumerate}

Therefore, Equation~\eqref{eqn:cov-concentration-fixed-u} hold for a fixed $U \in \R^{d\times d}$ with probability at least $1 - \exp(-\Omega(n \delta^2))$.

There is a set $\CC$ of $k^2$-sparse matrices with unit Frobenius norm with size $|\CC| = d^{O(k^2)}$,
such that for any $Y \in \R^{d \times d}$, there exists a matrix $U \in \CC$ such that $U \bullet Y \ge (1/2) \fksnorm{Y}$.
Taking a union bound over $\CC$, we know that Equation~\eqref{eqn:cov-concentration-fkk} holds with probability at least $1 - \exp(O(k^2 \log d) - \Omega(n \delta^2))$, which then immediately implies $\fkknorm{\Sigma_w - I} \le O(\delta)$.
\end{enumerate}

Taking a union bound over Part {\em (i)} and {\em (ii)}, we have that $G^\star$ is $(k,\eps,\delta)$-stable with probability at least $1 - \exp(O(k^2 \log d) - \Omega(n \delta^2))$.
Therefore, when $n = \Omega(k^2 \log d / \delta^2)$, $G^\star$ is $(k,\eps,\delta)$-stable with probability at least $1 - \exp(-\Omega(k^2 \log d))$.
\end{proof}

\section{Omitted Proofs in Section~\ref{sec:sparse-pca}}
\label{apx:sparse-pca}
In this section, we provide our main structural result for robust sparse PCA, which states that Algorithm~\ref{alg:sparsepca} works as long as the original good samples satisfy the stability conditions in Definition~\ref{def:sparse-pca-stability}.
The sample complexity and the stability conditions for sparse PCA will be proved in Appendix~\ref{apx:sparse-pca-stability}.

\smallskip
{\noindent \bf Theorem~\ref{thm:sparse-pca-struc}.~} {\em
Let $0 < \rho \le 1$, $0 < \eps < \eps_0$, and $\delta > \eps$.
Let $G^\star$ be a set of $n$ samples that is $(k,\eps,\delta)$-stable (as in Definition~\ref{def:sparse-pca-stability}) w.r.t. a centered distribution with covariance $I + \rho vv^\top$ for an unknown $k$-sparse unit vector $v \in \R^d$.
Let $S = (X_i)_{i=1}^n$ be an $\eps$-corrupted version of $G^\star$.
Then the output $u$ of Algorithm~\ref{alg:sparsepca} satisfies that $\fnorm{uu^\top - vv^\top} = O(\sqrt{\delta/\rho})$.
}

Recall that our objective function for sparse PCA is
\[
\min_w \; f(w) = \ftwoksnorm{M_w - I} \quad \textrm{ subject to } \; w \in \Delta_{n,\eps} \; ,
\]
where $M_w = \sum_i w_i X_i X_i^\top$ and $\ftwoksnorm{A} = \max_{Q \subseteq ([d] \times [d]), |Q| = 2k^2} \fnorm{A_Q}$.

We will in fact prove a stronger statement that any approximately optimal solution $w$ suffices for robust sparse PCA.
Formally, we show that the minimum objective value $\min_w f(w) \le \rho + \delta$, and given any $w \in \Delta_{n,\eps}$ with $f(w) \le \rho + O(\delta)$, Algorithm~\ref{alg:sparsepca} can achieve the guarantee stated in Theorem~\ref{thm:sparse-pca-struc}.

Throughout this section, we fix an approximately optimal solution $w \in \Delta_{n,\eps}$ with $f(w) \le \rho + O(\delta)$.
Let $A = M_w - I$.
Recall that $R$ is the support of $v v^\top$ and $Q$ is the largest $k^2$ entries of $A$.

\begin{proof}[Proof of Theorem~\ref{thm:sparse-pca-struc}]
We assume without loss of generality that $\rho = \Omega(\delta)$ for some sufficiently large constant.
Otherwise, the theorem holds vacuously because $\fnorm{uu^\top - vv^\top} \le 2 \le O(\sqrt{\delta/\rho})$.

Let $w^\star$ be the uniform distribution over the remaining good samples $G$.
By the stability conditions in Definition~\ref{def:sparse-pca-stability}, we can upper bound the objective value at $w^\star$.
\[
f(w^\star) = \ftwoksnorm{M_{w^\star} - I} \le \ftwoksnorm{M_{w^\star} - (I + \rho v v^\top)} + \ftwoksnorm{\rho v v^\top} \le \delta + \rho \; .
\]

We will show that given such any $w \in \Delta_{n,\eps}$ with $f(w) \le \rho + O(\delta)$, the output $u$ of Algorithm~\ref{alg:sparsepca} satisfies that $\fnorm{uu^\top - vv^\top} = O(\sqrt{\delta/\rho})$.

Recall that $u$ is the top eigenvector of $\bar{A_Q} = (A_Q + A_Q^\top)/2$. 
At a high level, we will show that $\bar{A_Q}$ is close to $\rho v v^\top$ and then use matrix perturbation theorem to show their top eigenvectors are close.

We will show in Lemma~\ref{lem:pca-vAQv-large} that $v^\top A_Q v \ge \rho - O(\delta)$.
Consequently, we can write $A_Q = \lambda v v^\top + B$, where $v^\top B v = 0$ and $\lambda \ge \rho - O(\delta)$.
Because $v^\top B v = 0$ and $\fnorm{A_Q} = f(w) \le \rho + O(\delta)$,
\[
\fnorm{B}^2 = \fnorm{A_Q}^2 - \lambda^2 \le (\rho + O(\delta))^2 - (\rho - O(\delta))^2 = O(\rho \delta + \delta^2) = O(\rho \delta) \; .
\]
Let $\bar B = (B + B^\top) / 2$.
We have $v^\top \bar B v = 0$, $\fnorm{\bar B} \le \fnorm{B} = O(\sqrt{\rho\delta})$, and
\[
\bar{A_Q} = \lambda v v^\top + \bar B \; .
\]
Notice that $u$ is the top eigenvector of $\bar{A_Q}$ and $v$ is the top eigenvector of $\rho v v^\top$.
By the matrix perturbation theorem (e.g., Davis-Kahan), we have
\[
\normtwo{u u^\top - v v^\top} = O\left(\frac{\normtwo{\bar B}}{\lambda - \lambda_2}\right) \le O\left(\frac{\sqrt{\rho\delta}}{\rho}\right) = O(\sqrt{\delta/\rho}) \; ,
\]
where $\lambda_2$ is the second largest eigenvalue of $\bar{A_Q}$.
The eigengap $\lambda - \lambda_2 = \Omega(\rho)$, because the top eigenvalue of $\bar{A_Q}$, which is at least $\lambda$, is close to its Frobenius norm.
More specifically, we can have say $\lambda \ge \rho - O(\delta) \ge \frac{8}{9}\rho$, and $\lambda_2 \le \frac{2}{3}\rho$ due to $\lambda^2 + \lambda_2^2 \le \fnorm{\bar{A_Q}}^2 \le \fnorm{A_Q}^2 \le (\rho + O(\delta))^2 \le (\frac{10}{9}\rho)^2$.

We conclude the proof by noticing that
\[
\fnorm{uu^\top - vv^\top} \le \sqrt{2} \normtwo{uu^\top - vv^\top} = O(\sqrt{\delta/\rho}) \; . \qedhere
\]
\end{proof}

We used the following lemma in the proof of Theorem~\ref{thm:sparse-pca-struc}, which intuitively states that $A_Q$ is close to $\rho v v^\top$ when measured by $v v^\top$.
\begin{lemma}
\label{lem:pca-vAQv-large}
Consider the same setting as in Theorem~\ref{thm:sparse-pca-struc}.  We have $v^\top A_Q v \ge \rho - O(\delta)$.
\end{lemma}
\begin{proof}
Let $A = \lambda v v^\top + B$ with $v^\top B v = 0$.
Recall that $w_G = \sum_{i \in G} w_i \ge \frac{1 - 2\eps}{1-\eps}$.

By the stability conditions in Definition~\ref{def:sparse-pca-stability}, we have
\begin{align*}
\lambda = v^\top A v
&= v^\top \left(\sum_{i=1}^n w_i X_i X_i^\top - I\right) v \\
&\ge v^\top \left(\sum_{i \in G} w_i X_i X_i^\top - I\right) v \\
&= v^\top \left(\sum_{i \in G} w_i \left(X_i X_i^\top - I - \rho v v^\top\right) \right) v - (1-w_G) + w_G \rho  \\
&\ge w_G \rho - \fksnorm{\sum_{i \in G} w_i (X_i X_i^\top - I - \rho v v^\top)} - (1-w_G) \\
&\ge \frac{1-2\eps}{1-\eps}\rho - \frac{1-\eps}{1-2\eps} O(\delta) - \frac{\eps}{1-\eps} \ge \rho - O(\delta) \; .
\end{align*}
The last inequality uses that $\eps < \delta$ and $\rho \le 1$.
Note that this also implies
\[
\fnorm{A_R} \ge v^\top A_R v = v^\top A_R v \ge \rho - O(\delta) \; .
\]

At a high level, we want to show that $v^\top A_Q v$ is close to $v^\top A_R v$.
Because
\[
v^\top A_Q v = v^\top A_R v - v^\top A_{R\setminus Q} v \; ,
\]
we will focus on the quadratic form $v^\top A_{R\setminus Q} v = \mathrm{vec}(A_{R\setminus Q})^\top \mathrm{vec}(v v^\top)$, where $\mathrm{vec}(\cdot)$ is the vectorization of a matrix.
We will upper bound this term by $\norminf{\mathrm{vec}(A_{R\setminus Q})} \normone{\mathrm{vec}((v v^\top)_{R \setminus Q})}$.

\begin{enumerate}[leftmargin=0.25in]
\item[\em (i)]
We first prove that every entry in $A_{R\setminus Q}$ has magnitude at most $O(\sqrt{\rho\delta}/k)$.

In particular, we will show that the smallest entry in $Q$ has magnitude $O(\sqrt{\rho\delta}/k)$.

Let $\bar R = ([d]\times [d])\setminus R$.
We have $\fksnorm{B_{\bar R}} = O(\sqrt{\rho\delta})$, otherwise $f(w) = \ftwoksnorm{A}^2 \ge \fnorm{A_R}^2 + \fksnorm{B_{\bar R}}^2$ would be larger than $\rho + O(\delta)$.

Observe that the $2k^2$-th largest entry of $A$ (i.e., the smallest entry of $A$ in $Q$) is upper bounded by the $k^2$-th largest entry of $A_{\bar R} = B_{\bar R}$, so its magnitude is at most $\frac{\fksnorm{B_{\bar R}}}{k} = O(\sqrt{\rho\delta}/k)$.

\item[\em(ii)]
Next we show that the average magnitude of $(v v^\top)_{R\setminus Q}$ is small.

Let $t = |R \setminus Q| > 0$. (If $t = 0$, then $Q = R$ and the lemma follows from previous calculations.)

Let
\[
r = \frac{\sum_{(i,j) \in R\setminus Q}\abs{v_i v_j}}{t}
\]
be the average magnitude of entries in $(v v^\top)_{R \setminus Q}$.
We will show that $r = O(\sqrt{\delta/(\rho t)})$.

Notice that $\fnorm{(\lambda v v^\top)_{R\setminus Q}} = \Omega(\rho r \sqrt{t})$ (because $\lambda = \Omega(\rho)$ and the Frobenius norm is minimized when all entries are equal) and $\fnorm{B_{R\setminus Q}} \le \fnorm{B_R} = O(\sqrt{\rho\delta})$ (otherwise $f(w) \ge \fnorm{A_R}^2 = \lambda^2 + \fnorm{B_R}^2$ would be larger than $\rho + O(\delta)$).

By the triangle inequality,
\[
\fnorm{A_{R\setminus Q}} = \fnorm{(\lambda v v^\top + B)_{R\setminus Q}}\ge \Omega(\rho r \sqrt{t}) - O(\sqrt{\rho\delta}) \; .
\]
On the other hand, by Part {\em (i)}, we know that every entry of $A_{R\setminus Q}$ is small,
\[
\fnorm{A_{R\setminus Q}} \le O(\sqrt{\rho\delta}/k \cdot \sqrt{t})
\]

Putting the above two inequalities together and solving for $r$, we get
\[
r \le \frac{\sqrt{\delta/\rho}}{k} + \frac{\sqrt{\delta/\rho}}{\sqrt{t}} = O\left(\sqrt{\frac{\delta/\rho}{t}}\right) \; .
\]
The last step uses $t \le |R| = k^2$.
\end{enumerate}

Finally, we can upper bound $v^\top A_{R\setminus Q} v$ by
\begin{align*}
v^\top A_{R\setminus Q} v
&= \mathrm{vec}(A_{R\setminus Q})^\top \mathrm{vec}((vv^\top)_{R\setminus Q}) \\
&\le \norminf{\mathrm{vec}(A_{R\setminus Q})} \normone{\mathrm{vec}((vv^\top)_{R\setminus Q})} \\
&\le O\left(\frac{\sqrt{\rho\delta}}{k}\right) \cdot (r t)
\le O\left(\frac{\sqrt{\rho\delta}}{k}\cdot \sqrt{\frac{\delta/\rho}{t}} \cdot t \right) = O(\delta) \; .
\end{align*}
The lemma follows immediately because
\[
v^\top A_Q v = v^\top A_R v - v^\top A_{R\setminus Q} v \ge \rho - O(\delta) - O(\delta) = \rho - O(\delta) \; . \qedhere
\]
\end{proof}

\subsection{Stability Conditions for Robust Sparse PCA}
\label{apx:sparse-pca-stability}

\smallskip
{\noindent \bf Lemma~\ref{lem:sparse-pca-stability}.~} {\em
Let $0 < \rho \le 1$ and $0 < \eps < \eps_0$.
Let $D$ be a centered subgaussian distribution with covariance $I + \rho vv^\top$ for a $k$-sparse unit vector $v \in \R^d$.
Let $G^\star$ be a set of $n = \Omega(k^2 \log d/\delta^2)$ samples drawn from $D$. Then then with probability at least $1-\exp(-\Omega(k^2\log d))$, $G^\star$ is $(k, \eps, \delta)$-stable (as in Definition~\ref{def:sparse-pca-stability}) w.r.t. $D$ for $\delta = O(\eps\log(1/\eps))$.
}

The proof of Lemma~\ref{lem:sparse-pca-stability} is almost identical to Part {\em (ii)} of Lemma~\ref{lem:sparse-mean-stability}.
We give a proof sketch highlighting the differences.

Notice that the PCA stability conditions are only on the second moment, and the $\delta$ in the PCA stability conditions plays the role of the ``$\delta^2/\eps$'' in the second-moment stability conditions for sparse mean.

Similar to the proof Lemma~\ref{lem:sparse-pca-stability}, it is sufficient to upper bound the norm with $O(\delta)$ and for all vertices of $\Delta_{n,\eps}$.
Or equivalently,
\begin{align*}
\ftwoksnorm{\frac{1}{n} \sum_{i \in G^\star} \left(X_i X_i^\top - I - \rho vv^\top \right)} &= O(\delta) \; \text{ and } \\
\ftwoksnorm{\frac{1}{n} \sum_{i \in L} \left(X_i X_i^\top - I - \rho vv^\top \right)} &= O(\delta) \; \text{ for all } \; |L|=\eps n \; .
\end{align*}

Fix some $U \in \R^{d \times d}$ with $\fnorm{U} = 1$.

Notice that since $0 < \rho \le 1$, the Hanson-Wright inequality continues to hold when the covariance matrix is $\Sigma = I + \rho vv^\top \preceq 2 I$.
\[
\Pr\left[U \bullet \left(\frac{1}{n} \sum_{i \in G^\star} \left(X_i X_i^\top - (I + \rho v v^\top) \right)\right) > \delta \right] \le \exp(-\Omega(n \delta^2)) \; .
\]

By hypercontractivity that $U \bullet (X X^\top - (I + \rho v v^\top))$ has exponential tails.
Consequently, we can show that with probability at least $1-\exp(-\Omega(n \delta))$, 
\[
U \bullet \left(\frac{1}{n}\sum_{i \in L} (X_i X_i^\top - (I + \rho v v^\top))\right) \le O(\delta)
\]
for all $|L| = \eps n$.
Therefore, the desired conditions hold for a fixed $U$ with probability at least $1 - \exp(-\Omega(n \delta^2))$.

We can then take a union bound over a net $|\CC|$ of $2k^2$-sparse matrices $U$ of size $|\CC| = d^{O(k^2)}$ to show that, when $n = \Omega(k^2 \log d/\delta^2)$,
\[
\ftwoksnorm{\sum_{i \in G^\star \setminus L} \frac{1}{(1-\eps)n} X_i X_i^\top - (I + \rho v v^\top)} \le O(\delta) \quad \text{ for all $|L| = \eps n$} \; ,
\]
with probability at least $1 - \exp(\Omega(-k^2 \log d))$.

\section{Algorithmic Results: Finding Stationary Points}
\label{apx:optimization}
In this section, we present our algorithmic results for robust sparse mean estimation and robust sparse PCA.
We show that one can find an approximate stationary point that suffices for the underlying robust estimation problem in a polynomial number of iterations.

When the true distribution is subgaussian, we prove that projected gradient descent can compute a good stationary point in $\tilde O(d^4/\eps^2)$ iterations for robust sparse mean estimation, and in $\tilde O(n d^2/\eps)$ iterations for robust sparse PCA.

We note that our iteration complexity is fairly loose and we did not make an effort to optimize the polynomial dependence.~\footnote{We believe the iteration complexity of robust sparse mean estimation can be improved if we run mirror descent to minimize $f$~(similar to the way~\cite{HopkinsLZ20} improved the iteration complexity of~\cite{ChengDGS20} for the non-sparse case).}

\begin{theorem}
\label{thm:sparse-mean-algo}
Fix $k > 0$ and $0 < \eps < \eps_0$. Let $S = (X_i)_{i=1}^n$ be an $\eps$-corrupted set of $n = \Omega(k^2\log d/\eps^2)$ samples drawn from a subgaussian distribution with unknown mean $\mu \in \R^d$ and covariance $I$.
Consider the optimization problem $\min_{w \in \Delta_{n,\eps}} f(w)$ where $f(w) = \fkknorm{\Sigma_w - I}$.
After $\tilde O(d^4/\eps^2)$ iterations, projected subgradient descent can output $w \in \Delta_{n,\eps}$ such that, with high probability, $\normtwok{\mu_w - \mu} = O(\eps\sqrt{\log(1/\eps)})$.
\end{theorem}

\begin{theorem}
\label{thm:sparse-pca-algo}
Let $0 < \rho \le 1$ and $0 < \eps < \eps_0$.
Let $S = (X_i)_{i=1}^n$ be an $\eps$-corrupted set of $n = \Omega(k^2\log d/\eps^2)$ samples drawn from a centered subgaussian distribution with covariance $I + \rho vv^\top$ for an unknown $k$-sparse unit vector $v \in \R^d$.
Consider the optimization problem $\min_{w \in \Delta_{n,\eps}} f(w)$ where $f(w) = \ftwoksnorm{M_w - I}$.
After $\tilde O(d^2/\eps)$ iterations, projected subgradient descent can output $w \in \Delta_{n,\eps}$ such that, with high probability, Algorithm~\ref{alg:sparsepca} can obtain $u \in \R^d$ from $w$ such that $\fnorm{uu^\top - vv^\top} = O(\sqrt{\eps\log(1/\eps)/\rho})$.
\end{theorem}

\subsection{Algorithmic Results: Robust Sparse Mean Estimation}
Recall the objective function $f(w) = \fkknorm{\Sigma_w - I}$ for robust sparse mean estimation.

Note that $f(w)$ may not be differentiable.
To circumvent this, we view $\min_w f(w)$ as a minimax optimization problem:
\[
\min_{w\in\Delta_{n,\eps}}f(w) = \min_{w\in\Delta_{n,\eps}}\max_{Y \in \YY} F(w, Y) \quad \text{ where } \quad F(w, Y) = (\Sigma_w - I) \bullet Y \; ,
\]
and $\YY = \{Y \in \R^{d \times d} : \fnorm{Y} = 1$ and $Y$ is non-zero in at most $k$ rows and $k$ entries in each row$\}$.

We use projected subgradient descent (PGD) to minimize $f(w) = \max_Y F(w, Y)$ (a formal description of PGD is given in Lemma~\ref{lem:pgd}).
In each iteration, we first compute a matrix $Y$ that maximizes $F(w, Y)$ for the current $w$:
Let $Q$ denote the set of $k^2$ entries that maximizes $\fnorm{(\Sigma_w - I)_Q}$, with the constraint that $Q$ contains entries from $k$ rows with $k$ entries in each row (breaking ties arbitrarily).
\[
f(w) = \fkknorm{\Sigma_w - I} = (\Sigma - I) \bullet Y \quad \text{ where } \quad Y = \frac{(\Sigma_w - I)_Q}{\fnorm{(\Sigma_w - I)_Q}} \; .
\]
We then run (one iteration of) PGD to update $w$ using the subgradient $\nabla_w F(w, Y)$:
\begin{align*}
w &\leftarrow \PP_{\Delta_{n,\eps}}(w - \eta \nabla_w F(w, Y)) \\
\text{where } \; \nabla_w F(w, Y) &= \diag(X^\top Y X) - X^\top (Y + Y^\top) X w \; ,
\end{align*}
$\eta$ is the step size of PGD that we will decide later, and $\PP_\KK(\cdot)$ is the $\ell_2$ projection operator onto $\KK$.

Because $f$ may not be differentiable, we cannot use the notion of stationarity in Definition~\ref{def:stationary}.
Instead, to prove Theorem~\ref{thm:sparse-mean-algo}, we show that after we run PGD for a sufficient number of iterations, a different kind of approximate stationarity holds.
For this notion of approximate stationarity, we need to work with a smoothed variant of the objective function known as the Moreau envelope.

\begin{definition}[Moreau Envelope]
\label{def:moreau}
For any function $f$ and closed convex set $\KK$, its associated Moreau envelope $f_\beta(w)$ is defined as
\[
f_\beta(w) := \min_{\tilde w \in \KK} f(\tilde w) + \beta \normtwo{w - \tilde w}^2 \; .
\]
\end{definition}
The Moreau envelope can be thought of as a form of convolution between the original function $f$ and a quadratic, so as to smoothen the landscape.
In particular, when $f(w)$ takes the form of a maximization problem $f(w) = \max_Y F(w, Y)$ with $F$ a mapping that is $\beta$-smooth in the $w$ parameter, the Moreau envelope is also $\beta$-smooth.
Therefore, the approximate stationarity of the Moreau envelop can be directly defined through its gradient.

To continue, we state a result from~\cite{ChengDGS20} to prove this form of approximate stationarity holds.
We omit the proof of Lemma~\ref{lem:pgd}, which generalizes the analysis in recent work (e.g.,~\cite{davis2018stochastic}) that provides convergence guarantees for weakly convex optimization problems.

\begin{lemma}[Lemma 4.2 of~\cite{ChengDGS20}]
\label{lem:pgd}
Let $\KK$ be a closed convex set.
Let $F(w, Y)$ be a function which is $L$-Lipschitz and $\beta$-smooth with respect to $w$.
Consider the following optimization problem $\min_{w \in \KK} f(w)$ where $f(w) = \max_{Y \in \YY} F(w, Y)$. 

Starting from any initial point $w_0 \in \KK$, we run iterative updates of the form:
\begin{align*}
Y_\tau &= \arg\max_{Y \in \YY} F(w_\tau, Y) \\
w_{\tau + 1} &= \PP_\KK(w_\tau - \eta \nabla_w F(w_\tau, Y_\tau))
\end{align*}
for $T$ iterations with step size $\eta = \frac{\xi}{\sqrt{T}}$.
Then, we have
\[
\min_{0 \le \tau < T} \normtwo{\nabla f_\beta(w_\tau)}^2 \le \frac{2}{\sqrt{T}}\left(\frac{f_\beta(w_0) - \min_w f(w)}{\xi} + \xi \beta L^2 \right)
\]
where $f_\beta(w)$ is the Moreau envelope of $f$ as in Definition~\ref{def:moreau}.
\end{lemma}

In our setting, we have $f(w) = \max_{Y \in \YY} F(w, Y)$ with $F(w, Y) = Y \bullet (\Sigma_w - I)$.
We will show in Lemma~\ref{lem:sparse-mean-beta-L} that $F(w, Y)$ obeys the required Lipschitz and smoothness properties.

We are now ready to prove Theorem~\ref{thm:sparse-mean-algo}. 

\begin{proof}[Proof of Theorem~\ref{thm:sparse-mean-algo}]
Note that when $n = \Omega(k^2\log d/\eps^2)$, the original good samples $G^\star$ is $(k,\eps,\delta)$-stable for $\delta = O(\eps\sqrt{\log(1/\eps)})$ with probability at least $1 - \exp(-\Omega(k^2 \log d))$.

In addition, we can assume without loss of generality that $\normtwo{X_i} \le O(\sqrt{d \log d})$ for all $i \in S$.
We can throw away samples in $S$ that are $\Omega(\sqrt{d \log d})$-far from the empirical median, since with high probability, all good samples are $O(\sqrt{d \log d})$-close to the empirical median. 
Then we shift all samples by the empirical median, which does not affect the final error guarantee $\normtwok{\mu_w - \mu}$.

We will show in Lemma~\ref{lem:sparse-mean-beta-L} that $F(w, Y)$ is $L$-Lipschitz and $\beta$-smoothness with $L = \tilde O(\sqrt{n} d)$ and $\beta = \tilde O(nd)$.
In addition, we have $B := f_\beta(w_0) - \min_w f(w) \le f_\beta(w_0) \le f(w_0) \le \tilde O(d)$, and $\gamma = O(n^{1/2} \delta^2 \eps^{-3/2})$ from Theorem~\ref{thm:sparse-mean-struc}.

Therefore, we can apply Lemma~\ref{lem:pgd} with $\KK = \Delta_{n,\eps}$ to obtain that after $T \ge O( B \beta L^2 / \gamma^4) = \tilde O(d^4 / \eps^2)$ iterations, we have $\normtwo{\nabla f_\beta(w_\tau)} \le \gamma$.

The condition $\normtwo{\nabla f_\beta(w)} \le \gamma$ implies that there exists a vector $\hat w$ such that
\[
\normtwo{\hat w - w} = \frac{\gamma}{2\beta} \quad \text{ and } \quad \min_{g \in \partial f(\hat w) + \partial \II_\KK(\hat w)} \normtwo{g} \le \gamma \; .
\]

We first show that $\hat w$ is a good solution.
We note that a similar argument was used in~\cite{ChengDGS20} for working with Moreau envelope of the spectral norm.

It is well-known that the subdifferential of the support function is the normal cone, which is in turn the polar of the tangent cone.
That is,
\[
\partial \II_{\KK}(\hat w) = \NN_{\KK}(\hat w) = (\CC_{\KK}(\hat w))^\circ \; .
\]
Thus, there exists a vector $g = \nu + v$ with $\normtwo{g} \le \gamma$ such that 
$\nu \in \partial f(\hat w)$ and $v \in (\CC_{\KK}(\hat w))^\circ$.
Now consider any unit vector $u \in \CC_{\KK}(\hat w)$:
\[
-\gamma \le u^\top g = u^\top \nu + u^\top v \le u^\top \nu \;,
\]
where the last step follows from the definition of the polar set.
In other words, there exists a vector $\nu \in \partial f(\hat w)$ such that
\begin{equation}
\label{eqn:u-subgradient}
- \nu^\top u \le \gamma \quad \text{ for all unit vectors } \; u \in \CC_{\KK}(\hat w) \;.
\end{equation}

This is the notion of stationarity in Definition~\ref{def:stationary} which is used in Theorem~\ref{thm:sparse-mean-struc}.
Because $G^\star$ is $(k,\eps,\delta)$-stable, by Theorem~\ref{thm:sparse-mean-struc}, we must have $\normtwo{\mu_{\hat w} - \mu} \le O(\delta)$.

We conclude the proof by noticing that $w$ is very close to $\hat w$, 
so if $\hat w$ is a good solution, then $w$ must also be a good solution:
\begin{align*}
\normtwo{\mu_w - \mus}
&\le \normtwo{\mu_w - \mu_{\hat w}} + \normtwo{\mu_{\hat w} - \mus} \\
&\le \normtwo{X} \normtwo{w - \hat w} + O(\delta)
= O(\delta) \; .
\end{align*}
The last step uses $\normtwo{X} \normtwo{\hat w - w} = \sqrt{n} \max_i \normtwo{X_i} \cdot O\left(\gamma/\beta\right) = \tilde O(\delta^2 d^{-1/2} \eps^{-3/2}) = O(\delta)$.
\end{proof}

Lemma~\ref{lem:sparse-mean-beta-L} upper bounds the Lipschitzness and smoothness parameters of the function $F(w, Y)$ with respect to $w$.
\begin{lemma}
\label{lem:sparse-mean-beta-L}
Fix a set of samples $(X_i)_{i=1}^n$ with $\max_i \normtwo{X_i} = \tilde O(\sqrt{d})$.
Fix some $Y \in \R^{d \times d}$ with $\fnorm{Y} = 1$.
The function $F(w, Y) = (\Sigma_w - I) \bullet Y$ defined over $w \in \Delta_{n,\eps}$ is $L$-Lipschitz and $\beta$-smooth with respect to $w$ for $L = \tilde O(\sqrt{n} d)$ and $\beta = \tilde O(n d)$.
\end{lemma}
\begin{proof}
Recall that $X \in \R^{d \times n}$ is the sample matrix whose $i$-th column is $X_i$.


Fix any $w \in \Delta_{n,\eps}$.

Recall that
\[
\nabla_w F(w, Y) = \diag(X^\top Y X) - X^\top (Y + Y^\top) X w \; .
\]
Therefore,
\begin{align*}
\normtwo{\nabla_w F(w, Y)}
& \le \normtwo{\diag(X^\top Y X)} + \normtwo{X^\top (Y + Y^\top) X w} \\
& \le \sqrt{n} \max_i X_i^\top Y X_i + 2 \normtwo{X}^2 \normtwo{Y} \normtwo{w} \\
& \le \sqrt{n} \max_i \normtwo{X_i}^2 \normtwo{Y} + 2 n (\max_i \normtwo{X_i}^2) \normtwo{Y} \normtwo{w} \\
& \le \tilde O(\sqrt{n} d) \; .
\end{align*}
The last inequality uses the fact that $\max_i \normtwo{X_i} = \tilde O(\sqrt{d})$, $\normtwo{Y} \le \fnorm{Y} \le 1$, and $\normtwo{w} \le \sqrt{n} \norminf{w} = O(1/\sqrt{n})$.

For the smoothness parameter, note that
\[
\nabla^2_w F(w, Y) = -X^\top (Y + Y^\top) X \; .
\]
Thus,
\[
\normtwo{\nabla^2_w F(w, Y)} \le 2 \normtwo{X}^2 \normtwo{Y} \le n(\max_i \normtwo{X_i}^2) = \tilde O(n d) \; .
\]
This concludes that $L = \tilde O(\sqrt{n} d)$ and $\beta = \tilde O(n d)$.
\end{proof}

\subsection{Algorithmic Results: Robust Sparse PCA}

Recall that for robust sparse PCA, our objective function is $f(w) = \ftwoksnorm{M_w - I}$ where $M_w = \sum_i w_i X_i X_i^\top$.

Note that $f(w)$ is convex in $w$.
Therefore, we can obtain an upper bound on the number of iterations from well-known results on the convergence of projected subgradient descent for $L$-Lipschitz convex functions (see, e.g.,~\cite{bubeck2014convex}).

\begin{proof}[Proof of Theorem~\ref{thm:sparse-pca-algo}]
Similar to the proof of Theorem~\ref{thm:sparse-mean-algo}, we assume without loss of generality that $G^\star$ is $(k,\eps,\delta)$-stable for $\delta = O(\eps\log(1/\eps))$ and $\max_{i\in S}\normtwo{X_i}=\tilde O(\sqrt{d})$, which happens with probability at least $1 - \exp(-\Omega(k^2 \log d))$.

We will show in Lemma~\ref{lem:L-pca} that $f(w)$ is $L$-Lipschitz for $L = \tilde O(\sqrt{n} d)$.
In the proof of Theorem~\ref{thm:sparse-pca-struc} in Appendix~\ref{apx:sparse-pca}, we know that there exists $w \in \Delta_{n,\eps}$ with $f(w) \le \rho + \delta$, and moreover, given the stability of $G^\star$, it is sufficient to find a weight vector $w$ with $f(w) \le \rho + O(\delta)$ to obtain the claimed error guarantee $\fnorm{uu^\top - vv^\top} = O(\sqrt{\eps\log(1/\eps)/\rho})$.

Therefore, it is sufficient to compute a solution $w$ with $f(w) - \min_w f(w) \le O(\delta)$.
It is well-known that (e.g., Theorem 3.2 in~\cite{bubeck2014convex}), the number of iterations required to compute such $w$ is $T \ge O(R^2 L^2 / \delta^2)$, where $R$ is the radius of the feasible region $\Delta_{n,\eps}$, and $L$ is the Lipschitz parameter of $f$.
Plugging in $R = O(\sqrt{\eps/n})$ and $L = \tilde O(\sqrt{n} d)$, the iteration complexity is $T \ge \tilde O(\frac{\eps}{n} \cdot n d^2 \cdot \delta^{-2}) = \tilde O(d^2 / \eps)$.
\end{proof}

\begin{lemma}
\label{lem:L-pca}
Fix a set of samples $(X_i)_{i=1}^n$ with $\max_i \normtwo{X_i} = \tilde O(\sqrt{d})$.
Let $M_w = \sum_{i=1}^n w_i X_i X_i^\top$.
The function $f(w) = \ftwoksnorm{M_w - I}$ defined over $w \in \Delta_{n,\eps}$ is $L$-Lipschitz for $L = \tilde O(\sqrt{n} d)$.
\end{lemma}
\begin{proof}
By the definition of the $\ftwoksnorm{\cdot}$ norm, we can define
\[
f(w) = \max_{Y \in \YY} F(w, Y) \quad \text{ where } \quad F(w, Y) = (M_w - I) \bullet Y
\]
and $\YY = \{Y \in \R^{d \times d} : \fnorm{Y} = 1$ and $Y$ has at most $2k^2$ non-zeros$\}$.

Because the maximum of $L$-Lipschitz functions is still $L$-Lipschitz, it is sufficient to upper bound the Lipschitz parameter of $F(w, Y)$ for fixed $Y$.

We have
\[
\nabla_w F(w, Y) = \diag(X^\top Y X) \; .
\]
Therefore,
\[
\normtwo{\nabla_w F(w, Y)} \le \sqrt{n} \max_i X_i^\top Y X_i \; \le \sqrt{n} \max_i \normtwo{X_i}^2 \normtwo{Y} \le \tilde O(\sqrt{n} d) \; .
\]
This concludes that $L = \tilde O(\sqrt{n} d)$.
\end{proof}

\section{Simpler Analysis for Robust Mean Estimation via Gradient Descent}

Our analysis in~Section~\ref{sec:sparse-mean} can be applied almost directly to general (non-sparse) robust mean estimation.
We present the corresponding structural results and proofs in this section.
As for the objective function, we will instead use the spectral norm $f(w) = \normtwo{\Sigma_w - I}$.
We can obtain the main result of~\cite{ChengDGS20} (that natural non-convex formulations of robust mean estimation has no bad stationary points) and greatly simplify their analysis.

The main advantages of our analysis include: (1) our structural result holds under broader distributional assumptions (e.g., subgaussian and bounded covariance distributions), (2) our analysis is shorter and conceptually simpler, and (3) we show that any $\gamma$-approximate stationary point $w$ suffices for robust mean estimation for a larger value of $\gamma$, so a good $w$ can be found with fewer iterations.

\begin{theorem}
\label{thm:mean-struc}
Fix $0 < \eps < \eps_0$ and $\delta > \eps$.
Let $G^\star$ be a set of $n$ samples that is $(\eps,\delta)$-stable (as in Definition~\ref{def:stability}) with respect to a $d$-dimensional ground-truth distribution with unknown mean $\mu$.
Let $S = (X_i)_{i=1}^n$ be an $\eps$-corrupted version of $G^\star$.
Let $f(w) = \normtwo{\Sigma_w - I}$.
Let $\gamma = O(n^{1/2} \delta^2 \eps^{-3/2})$.
Then, for any $w \in \Delta_{n,\eps}$ that is a $\gamma$-stationary point of $f(w)$, we have $\normtwo{\mu_w - \mu} = O(\delta)$.
\end{theorem}

When the ground-truth distribution is a spherical Gaussian $\NN(\mu, I)$ as in~\cite{ChengDGS20}, our sample complexity $n = \tilde \Omega(d/\eps^2)$ and error guarantee $\delta = O(\eps\sqrt{\log(1/\eps)}$ both match those in prior works (which are optimal up to logarithmic factors).
Moreover, our result (Theorem~\ref{thm:mean-struc}) states that any $\gamma$-stationary point suffices for $\gamma = O(n^{1/2} \delta^2 \eps^{-3/2}) = O(n^{1/2} \eps^{1/2} \log(1/\eps))$, which is $\sqrt{\eps n}$ times larger than the $\gamma = O(\log(1/\eps))$ in previous work~\cite{ChengDGS20}.

\paragraph{Overview.}
The high-level idea of our proof is identical to that in Section~\ref{sec:sparse-mean}: Given a weight vector $w$, we show that if $w$ is not a good solution, then moving toward $w^\star$ (the uniform distribution on the good input samples) will decrease the objective value.
The changes in the proofs are mostly in using different norms and not having to handle sparsity.
Formally, by Lemma~\ref{lem:sample-cov-mixed-w}, we get
\[
\Sigma_{(1-\eta)w + \eta w^\star} = (1-\eta)\Sigma_w + \eta\Sigma_{w^\star} + \eta(1-\eta)(\mu_w - \mu_{w^\star})(\mu_w - \mu_{w^\star})^\top \; .
\]
We can then take spectral norm on both sides (after subtracting the identity matrix) and show that the third term can be essentially ignored.
Because $w$ is a bad solution and $w^\star$ is a good solution, $\normtwo{\Sigma_w - I}$ must be much larger than $\normtwo{\Sigma_{w^\star} - I}$, so the objective function must decrease when we move from $w$ to $(1-\eta)w + \eta w^\star$.

We start with the stability conditions we need for (non-sparse) robust mean estimation.

\paragraph{Deterministic Conditions.}
For (non-sparse) robust mean estimation, we require the following deterministic conditions (Definition~\ref{def:stability}) on the original set of good samples $G^\star$.
We refer to these conditions as \emph{stability conditions} because, at a high level, they state that the first and second moments of the good samples are stable when a small fraction of the samples are removed.

\begin{definition}[Stability Conditions]
\label{def:stability}
A set of $n$ samples $G^\star = (X_i)_{i=1}^n$ is said to be $(\eps,\delta)$-stable (with respect to a distribution with mean $\mu$) iff for any weight vector $w \in \Delta_{n,2\eps}$, we have
\begin{align*}
\normtwo{\mu_w - \mu} \le \delta
\quad \text{ and } \quad
\normtwo{\Sigma_w - I} \le \delta^2/\eps \; .
\end{align*}
\end{definition}

\paragraph{Key Lemma.}
We use Lemma~\ref{lem:cross-term-psd-le} to relate the spectral norm of the rank-one matrix $(\mu_w - \mu_{w^\star})(\mu_w - \mu_{w^\star})^\top$ to that of $(\Sigma_w - I)$, showing that we can essentially ignore this term.

\begin{lemma}
\label{lem:cross-term-psd-le}
Let $G^\star$ be an $(\eps,\delta)$-stable set of $n$ samples with $0 < \eps \le \delta$.
Let $S$ be an $\eps$-corrupted version of $G^\star$.
Then, for any $w \in \Delta_{n,\eps}$, we have
\[
\normtwo{\mu_w - \mu_{w_\star}}^2 \le 4 \eps \left( \normtwo{\Sigma_w - I} + O(\tfrac{\delta^2}{\eps}) \right) \; .
\]
\end{lemma}
\begin{proof}
Recall that $S = G \cup B$ where $G$ is the set of (remaining) good samples and $B$ is the set of corrupted samples.
Let $\alpha = \normone{w_G}$ and $\beta = \normone{w_B}$.
Let $\bar{w} = w_G / \alpha$ and $\hat{w} = w_B / \beta$ denote the normalized version of $w_G$ and $w_B$.

We can write $w = \alpha \bar{w} + \beta \hat{w}$, by Lemma~\ref{lem:sample-cov-mixed-w}, we know that
\begin{align}
\label{eqn:mean-i}
\Sigma_w &= \alpha \Sigma_{\bar{w}} + \beta \Sigma_{\hat{w}} + \alpha \beta (\mu_{\bar{w}} - \mu_{\hat{w}})(\mu_{\bar{w}} - \mu_{\hat{w}})^\top \; .
\end{align}

Since $\beta \le \norminf{w} \cdot |B| \le \frac{\eps}{1-\eps}$, we have $\norminf{\bar w} = \frac{\norminf{w_G}}{\alpha} \le \frac{1}{(1-\eps)n} \cdot \frac{1}{1-\beta} \le \frac{1}{(1-2\eps)n}$.
Because $G^\star$ is $(\eps,\delta)$-stable and we can view $\bar w \in \Delta_{n,2\eps}$ as a weight vector on $G^\star$, by the stability conditions in Definition~\ref{def:stability},
\begin{align}
\label{eqn:mean-ii}
\normtwo{\Sigma_{\bar w} - I} &\le \tfrac{\delta^2}{\eps} \; .
\end{align}

Using Equations~\eqref{eqn:mean-i}~and~\eqref{eqn:mean-ii} and that $\Sigma_{\hat w} \succeq 0$, for any unit vector $v \in \R^d$,
\begin{align}
\label{eqn:mean-iii}
\begin{split}
\normtwo{\Sigma_w - I}
\ge v^\top (\Sigma_w - I) v
&= \alpha v^\top \Sigma_{\bar{w}} v + \beta v^\top \Sigma_{\hat{w}} v + \alpha \beta \left( (\mu_{\bar w} - \mu_{\hat w})^\top v \right)^2 - 1 \\
&\ge \alpha\left(1-\tfrac{\delta^2}{\eps}\right) - 1 + \alpha \beta \left( (\mu_{\bar w} - \mu_{\hat w})^\top v \right)^2 \; .
\end{split}
\end{align}

We know $\alpha (1-\delta^2/\eps)$ is close to $1$, so we are essentially left with only the last term on the right-hand side.
We will relate this term to $\normtwo{\mu_w - \mu_{w_\star}}^2$, which is what appears in the lemma statement.

Recall that $\alpha + \beta = 1$ and $w = \alpha \bar{w} + \beta \hat{w}$, and thus
\begin{align}
\label{eqn:mean-iv}
\beta(\mu_{\hat w} - \mu_{\bar w}) &= \beta \mu_{\hat w} + \alpha \mu_{\bar w} - \mu_{\bar w} = \mu_w - \mu_{\bar w} = (\mu_w - \mu_{w^\star}) + (\mu_{w^\star} - \mu_{\bar w}) \; .
\end{align}

Since $\bar w, w^\star \in \Delta_{n, 2\eps}$ and both only put positive weight on samples in $G$, it follows from the stability conditions (Definition~\ref{def:stability}) that
\begin{align}
\label{eqn:mean-v}
\abs{\left(\mu_{w^\star} - \mu_{\bar w} \right)^\top v}
&\le \normtwo{\mu_{w^\star} - \mu_{\bar w}}
\le \normtwo{\mu_{w^\star} - \mus} + \normtwo{\mus - \mu_{\bar w}}
\le 2 \delta \; .
\end{align}

We choose $v = \frac{\mu_w - \mu_{w^\star}}{\normtwo{\mu_w - \mu_{w^\star}}}$. 
From Equations~\eqref{eqn:mean-iv}~and~\eqref{eqn:mean-v}, we have
\begin{align}
\label{eqn:mean-vi}
\begin{split}
\left( \beta \cdot (\mu_{\hat w} - \mu_{\bar w})^\top v \right)^2
&= \left((\mu_w - \mu_{w^\star})^\top v + (\mu_{w^\star} - \mu_{\bar w})^\top v \right)^2 \\
&\ge \frac{\left( (\mu_{w} - \mu_{w_\star})^\top v \right)^2}{2} - \left( (\mu_{\bar w} - \mu_{w^\star})^\top v \right)^2
\ge \frac{\normtwo{\mu_{w} - \mu_{w_\star}}^2}{2} - 4 \delta^2 \; .
\end{split}
\end{align}
The first inequality in Equation~\eqref{eqn:mean-vi} uses the fact that $(x+y)^2 \ge \frac{x^2}{2} - y^2$ for any $x, y \in \R$.

Putting Equations~\eqref{eqn:mean-iii}~and~\eqref{eqn:mean-vi} together for our choice of $v$, we have
\begin{align*}
\normtwo{\Sigma_w - I}
&\ge \alpha \left(1-\tfrac{\delta^2}{\eps}\right) - 1 + \frac{\alpha}{\beta} \left( \beta \cdot(\mu_{\bar w} - \mu_{\hat w})^\top v \right)^2 \\
&\ge \frac{1-2\eps}{1-\eps} \left(1-\tfrac{\delta^2}{\eps}\right) - 1
  + \frac{1-2\eps}{\eps}\left(\frac{\normtwo{\mu_{w} - \mu_{w_\star}}^2}{2} - 4 \delta^2\right) \\
&\ge \frac{1}{4\eps} \normtwo{\mu_{w} - \mu_{w_\star}}^2 - O\left(\tfrac{\delta^2}{\eps}\right) \; . \qedhere
\end{align*}
\end{proof}

\paragraph{Proof of Theorem~\ref{thm:mean-struc}.}

We are now ready to prove our main result of this section.

\begin{proof}[Proof of Theorem~\ref{thm:mean-struc}]
Fix any weight vector $w \in \Delta_{n,\eps}$.
We will show that if $w$ is a bad solution, then $w$ cannot be an approximate first-order stationary point.

Let $c_1$ be the constant in $O(\cdot)$ in Lemma~\ref{lem:cross-term-psd-le}.
By Lemma~\ref{lem:cross-term-psd-le}, we know that if $\normtwo{\mu_w - \mus} \ge c_2 \delta$ for a sufficiently large constant $c_2$, then we have $\normtwo{\Sigma_w - I} \ge \frac{\normtwo{\mu_w - \mus}^2}{4\eps} - c_1 \delta^2 \ge (\frac{c_2^2}{4} - c_1)\frac{\delta^2}{\eps} = \Omega(\frac{\delta^2}{\eps})$.

Recall that $w^\star$ is the uniform distribution on $G$.
We will show that $f(w) = \normtwo{\Sigma_w - I}$ decreases if $w$ moves toward $w^\star$.
By Lemma~\ref{lem:sample-cov-mixed-w},
\begin{align*}
\Sigma_{(1-\eta)w + \eta w^\star} - I
&= (1-\eta)(\Sigma_w - I) + \eta(\Sigma_{w^\star} - I) + \eta(1-\eta)(\mu_w - \mu_{w^\star})(\mu_w - \mu_{w^\star})^\top \; .
\end{align*}
Using the triangle inequality for the spectral norm, we have
\begin{align*}
\normtwo{\Sigma_{(1-\eta)w + \eta w^\star}-I}
&\le (1-\eta)\normtwo{\Sigma_w - I} + \eta \normtwo{\Sigma_{w^\star} - I} + \eta(1-\eta) \normtwo{\mu_w - \mu_{w^\star}}^2 \; .
\end{align*}

We know that $\normtwo{\Sigma_{w^\star} - I} \le \delta^2/\eps$ by the stability conditions in Definition~\ref{def:stability}.
Using Lemma~\ref{lem:cross-term-psd-le} and that $f(w) = \normtwo{\Sigma_w - I} = \Omega(\frac{\delta^2}{\eps})$, we can show that for all $0 < \eta < 1$,
\begin{align}
\begin{split}
\label{eqn:mean-vii}
f((1-\eta)w + \eta w^\star)
&= \normtwo{\Sigma_{(1-\eta)w + \eta w^\star}-I} \\
&\approx \normtwo{(1-\eta) \left(\Sigma_w - I\right) + \eta \left(\Sigma_{w^\star} - I\right)} \\
&\le \normtwo{(1-\eta) \left(\Sigma_w - I\right)} \; + \; \normtwo{\eta \left(\Sigma_{w^\star} - I\right)} \\
&= (1-\eta)f(w) + \eta f(w^\star) < f(w) \; .
\end{split}
\end{align}
The last inequality requires $(\frac{1}{2}-4\eps)\normtwo{\Sigma_w - I} \ge (4c_1 + 1)\frac{\delta^2}{\eps}$, which holds if $\eps \le 1/10$ and we choose $c_2^2 \ge 164 \, c_1 + 40$.

It follows immediately that $w$ cannot be a stationary point of $f$.
Let $u = \frac{w^\star - w}{\normtwo{w^\star - w}}$ and $h = \eta \normtwo{w^\star - w}$.
When $h \to 0$, we have $w + h u = (1-\eta)w + \eta w^\star \in \Delta_{n, \eps}$ because $w, w^\star \in \Delta_{n, \eps}$ and $\Delta_{n, \eps}$ is convex.
Moreover, we have $\normtwo{w^\star - w} = O(\sqrt{\eps/n})$ and therefore
\[
\textstyle u^\top \nabla f(w) = \lim_{h \to 0} \frac{f(w + h u) - f(w)}{h} \le \lim_{\eta \to 0} \frac{-(\eta/2)f(w)}{\eta \normtwo{w^\star - w}} \le -\frac{\Omega(\delta^2 / \eps)}{\normtwo{w^\star - w}} \le -\Omega(n^{1/2} \delta^2 \eps^{-3/2}) \; .
\]
By Definition~\ref{def:stationary}, we know $w$ cannot be a $\gamma$-stationary point of $f$ for some $\gamma = O(n^{1/2} \delta^2 \eps^{-3/2})$.
\end{proof}

\section{Structural Results for Robust PCA}
In this section, we consider (non-sparse) robust PCA with spiked covariance.
In this model, the good samples are drawn from a centered subgaussian distribution with covariance $\Sigma = I + \rho v v^\top$ where $0 < \rho \le 1$ and $v \in \R^d$ is a unit vector.

We show that our idea in Section~\ref{sec:sparse-pca} for robust sparse PCA can be applied to the non-sparse case.
This leads to a new objective function~\eqref{eqn:pca-toptwo} and a simple analysis showing a similar structural result.
Consider the following optimization problem:
\begin{align}
\label{eqn:pca-toptwo}
\min_w f(w) &= \kftwonorm{M_w - I} \quad \textrm{ subject to } \; w \in \Delta_{n,\eps} \; ,
\end{align}
where $M_w = \sum_i w_i X_i X_i^\top$ and $\kftwonorm{A}$ is the Ky Fan $2$-norm of $A$ (i.e., the sum of the largest two singular values of $A$).

Because $(M_w - I)$ is always symmetric, we can equivalently define
\[
f(w) = \max_{\normtwo{u} = \normtwo{v} = 1, u \perp v} \abs{u^\top (M_w - I) u} + \abs{v^\top (M_w - I) v} \; .
\]
Note that $f(w)$ is convex in $w$.

We give some intuition for this objective function.
When $w = w^\star$, the uniform distribution on the good samples, $(M_w - I)$ is very close to $\rho v v^\top$.
To find a good weight vector with $M_w - I \approx \rho v v^\top$, we minimize the sum of (the absolute values of) the largest two eigenvalues of $(\Sigma_w - I)$.
We expect the top eigenvalue to be close to $\rho$, and consequently, the second largest eigenvalue must be small, which would ensure that $(M_w - I - \rho v v^\top) \approx 0$.


We show that any approximate stationary point $w$ of $f(w)$ yields a good solution for robust PCA.
In particular, the algorithm simply outputs the top eigenvector $u$ of $(M_w - I)$.

\begin{theorem}
\label{thm:pca-struc}
Let $0 < \rho \le 1$, $0 < \eps < \eps_0$, and $\delta > \eps$.
Let $G^\star$ be a set of $n$ samples that is $(\eps,\delta)$-stable (as in Definition~\ref{def:pca-stability}) w.r.t. a centered distribution with covariance $I + \rho vv^\top$ for an unknown unit vector $v \in \R^d$.
Let $S = (X_i)_{i=1}^n$ be an $\eps$-corrupted version of $G^\star$.

Let $f(w)$ be the objective function defined in Equation~\eqref{eqn:pca-toptwo}.
Then, there exists some $w^\star \in \Delta_{n,\eps}$ with $f(w^\star) \le \rho + \delta$.
Moreover, given any weight vector $w \in \Delta_{n,\eps}$ with $f(w) \le \rho + O(\delta)$, we can show that for the top eigenvector $u$ of $(M_w - I)$ satisfies that $\fnorm{uu^\top - vv^\top} = O(\delta/\rho)$.
\end{theorem}

Theorem~\ref{thm:pca-struc} requires the following deterministic conditions on the original good samples.

\begin{definition}
\label{def:pca-stability}
A set of $n$ samples $G^\star = (X_i)_{i=1}^n$ is said to be $(\eps,\delta)$-stable (with respect to a centered distribution with covariance $\Sigma = I + \rho v v^\top$) iff for any weight vector $w \in \Delta_{n,2\eps}$, we have
\begin{align*}
\kftwonorm{M_w - I} &\le \delta \; .
\end{align*}
where $M_w = \sum_i w_i X_i X_i^\top$ and $\kftwonorm{\cdot}$ is the Ky Fan $2$-norm.
\end{definition}

In particular, when the ground-truth distribution is subgaussian, a set of $n = \Omega(d/\delta^2)$ samples satisfies the stability conditions in~\ref{def:pca-stability} with high probability.

\begin{proof}[Proof of Theorem~\ref{thm:pca-struc}]
Recall that $w^\star$ is the uniform distribution on the remaining good samples.
By the stability conditions in Definition~\ref{def:pca-stability}, we know
\[
f(w^\star) = \kftwonorm{M_w - I} \le \kftwonorm{M_w - I - \rho v v^\top} + \kftwonorm{\rho v v^\top} \le \delta + \rho \; .
\]

Fix a $w \in \Delta_{n,\eps}$ with $f(w) \le \rho + O(\delta)$.
Let $A = M_w - I$.
We can write
\[
A = \lambda v v^\top + B
\]
where $v^\top B v = 0$.

Recall that $w_G = \sum_{i \in G} w_i \ge \frac{1 - 2\eps}{1-\eps}$.
By the stability conditions in Definition~\ref{def:pca-stability}, we know that $\normtwo{M_w - I - \rho v v^\top} \le \delta$ for all $w \in \Delta_{n,\eps}$, and therefore,
\begin{align*}
\lambda = v^\top A v
&= v^\top \left(\sum_{i=1}^n w_i X_i X_i^\top - I\right) v \\
&\ge v^\top \left(\sum_{i \in G} w_i X_i X_i^\top - I\right) v \\
&= v^\top \left(\sum_{i \in G} w_i \left(X_i X_i^\top - I - \rho v v^\top\right) \right) v - (1-w_G) + w_G \rho  \\
&\ge w_G \rho - \normtwo{\sum_{i \in G} w_i (X_i X_i^\top - I - \rho v v^\top)} - (1-w_G) \\
&\ge \frac{1-2\eps}{1-\eps}\rho - \frac{1-\eps}{1-2\eps} O(\delta) - \frac{\eps}{1-\eps} \ge \rho - O(\delta) \; .
\end{align*}

Consequently, we must have $\normtwo{B} \le O(\delta)$.
Otherwise, $f(w) = \kftwonorm{A} \ge \lambda + \normtwo{B}$ would be larger than $\rho + O(\delta)$.

By the matrix perturbation theorem, we have
\[
\fnorm{uu^\top - vv^\top} \le \sqrt{2} \normtwo{uu^\top - vv^\top} = O\left(\frac{\normtwo{B}}{\lambda - \lambda_2}\right) = O(\delta/\rho) \; .
\]
In the last step, $\lambda_2$ is the second largest eigenvalue of $A$, which is at most $\kftwonorm{A} - \lambda = O(\delta)$, so the eigengap is at least $\rho/2$.
\end{proof}

\end{document}